\documentclass{article} 

\usepackage[preprint]{colm2026_conference}

\usepackage{latexsym}

\usepackage[T1]{fontenc}

\usepackage[utf8]{inputenc}

\usepackage{microtype}
\usepackage{inconsolata}

\usepackage{latexsym}
\usepackage{algorithmic}
\usepackage{graphicx}
\usepackage{textcomp}
\usepackage{xcolor}
\usepackage{physics}
\usepackage{mathdots}
\usepackage{algorithm}
\usepackage{subfigure}
\usepackage{pgfplots}
\usepackage{tikz}
\usetikzlibrary{patterns}
\usepackage{multirow}
\usepackage{multicol}
\usepackage{array}
\usepackage{tabularx}    
\usepackage{wrapfig}     
\usepackage{adjustbox}   
\pgfplotsset{compat=1.18}

\usepackage{tcolorbox}
\tcbuselibrary{breakable}
\usepackage{listings}
\usepackage{booktabs}  
\usepackage[table]{xcolor}
\usepackage{placeins}
\usepackage{hyperref}
\usepackage{enumitem}

\usepackage{amsmath,amssymb,amsfonts}

\usepackage{amsthm}

\newtheorem{theorem}{Theorem}
\newtheorem{proposition}[theorem]{Proposition}
\newtheorem{corollary}[theorem]{Corollary}

\theoremstyle{definition}
\newtheorem{definition}[theorem]{Definition}

\theoremstyle{remark}

\let\cite\citep
\definecolor{darkblue}{rgb}{0, 0, 0.5}
\hypersetup{colorlinks=true, citecolor=darkblue, linkcolor=darkblue, urlcolor=darkblue}
\fancypagestyle{firstpage}{%
  \lhead{\raisebox{-0.28\height}{\includegraphics[height=18pt]{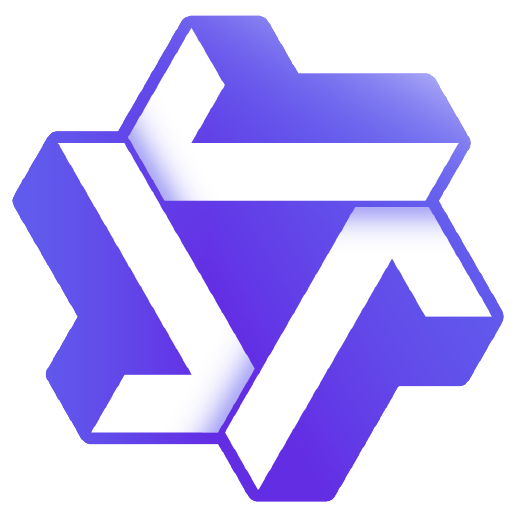}}\hspace{0.3em}{\large\bfseries AgentScope, Tongyi Lab}}%
  \rhead{}%
}

\title{R$^3$L: Reflect-then-Retry Reinforcement Learning with Language-Guided Exploration, Pivotal Credit, and Positive Amplification}

\author{
  Weijie Shi$^{\dagger,\S}$\thanks{Work done during internship at Tongyi Lab. Email: shiweijie0311@foxmail.com} \quad
  Yanxi Chen$^{\dagger}$ \quad
  Zexi Li$^{\dagger}$ \quad
  Xuchen Pan$^{\dagger}$ \\
  Yuchang Sun$^{\dagger}$ \quad
  Jiajie Xu$^{\ddagger}$ \quad
  Xiaofang Zhou$^{\S}$ \quad
  Yaliang Li$^{\dagger}$ \\[0.8em]
  $^{\dagger}$Tongyi Lab\,\raisebox{-0.3em}{\includegraphics[height=1.1em]{figure/logo.pdf}} \qquad $^{\ddagger}$Soochow University \\
  $^{\S}$Hong Kong University of Science and Technology
}

\begin{document}
\maketitle
\begin{abstract}
Reinforcement learning drives recent advances in LLM reasoning and agentic capabilities, yet current approaches struggle with both exploration and exploitation. Exploration suffers from low success rates on difficult tasks and high costs of repeated rollouts from scratch. Exploitation suffers from coarse credit assignment and training instability: Trajectory-level rewards penalize valid prefixes for later errors, and failure-dominated groups overwhelm the few positive signals, leaving optimization without constructive direction. To this end, we propose R$^3$L, Reflect-then-Retry Reinforcement Learning with Language-Guided Exploration, Pivotal Credit, and Positive Amplification. To synthesize high-quality trajectories, R$^3$L shifts from stochastic sampling to active synthesis via Reflect-then-Retry, leveraging language feedback to diagnose errors, transform failed attempts into successful ones, and reduce rollout costs by restarting from identified failure points. With errors diagnosed and localized, Pivotal Credit Assignment updates only the diverging suffix where contrastive signals exist, excluding the shared prefix from gradient update. Since failures dominate on difficult tasks and Reflect-then-Retry produces off-policy data, risking training instability, Positive Amplification upweights successful trajectories to ensure positive signals guide the optimization process. Experiments on agentic and reasoning tasks demonstrate 5\% to 52\% relative improvements over baselines while maintaining training stability. Our code is released at \url{https://github.com/shiweijiezero/R3L}. 
\end{abstract}

\section{Introduction}
Reinforcement learning (RL) like GRPO \cite{shao2024deepseekmath} drives recent advances in LLM reasoning and agentic capabilities \cite{cui2025enhancing,shi2025semantic,plaat2025agentic}. Recent systems such as DeepSeek-R1 \cite{guo2025deepseek}, DeepSeek-Math \cite{shao2024deepseekmath}, and Search-R1 \cite{jin2025search} demonstrate its effectiveness. However, as tasks scale to complex, multi-step agentic environments with sparse rewards, current approaches struggle with both exploration and exploitation.

\textbf{Inefficient exploration.} Stochastic sampling produces predominantly failed trajectories on difficult problems, leading to a lack of positive signals and even null gradients when all samples in a group fail \cite{nan2025ngrpo}. Repeated rollouts from scratch attempt to compensate but incur high cost. Furthermore, scalar rewards indicate correctness but provide no actionable guidance on why solutions failed or how to discover better ones \cite{scheurer2022training,scheurer2023training,zhang2025critique}. In practice, environments often provide abundant natural language feedback such as error messages, execution traces, and observation descriptions, yet current RL algorithms like GRPO cannot leverage this rich information. This calls for a guided exploration mechanism that leverages language feedback to efficiently synthesize successful trajectories. \looseness=-1

\textbf{Unstable exploitation.} Beyond discovering successful trajectories, learning from them poses distinct challenges. Trajectory-level credit assignment applies the same reward signal to all tokens regardless of where errors occur. When a trajectory fails due to a late mistake, valid earlier reasoning is suppressed alongside the error, introducing noise into gradient estimates. While Process Reward Models \cite{lightman2023let} offer step-level credit as an alternative, they require costly human annotation and often produce unstable supervision signals \cite{xiong2024watch}, calling for credit assignment methods that derive clear signals without external annotation. In group-relative methods, instability arises when failed trajectories dominate a group, as gradients become driven by error suppression rather than reinforcement of correct solutions. Suppressing errors reduces the probability of incorrect tokens, releasing probability mass that must be redistributed. Without sufficient positive signals to direct this mass toward correct solutions, it disperses across the vocabulary, driving the policy toward high entropy. We term this entropy collapse and analyze it in Appendix \ref{sec:entropy_collapse}. Worse still, synthesized off-policy data exacerbates this instability \cite{wu2025learning}.

\begin{wrapfigure}{r}{0.5\textwidth}
    \centering
    \includegraphics[width=\linewidth]{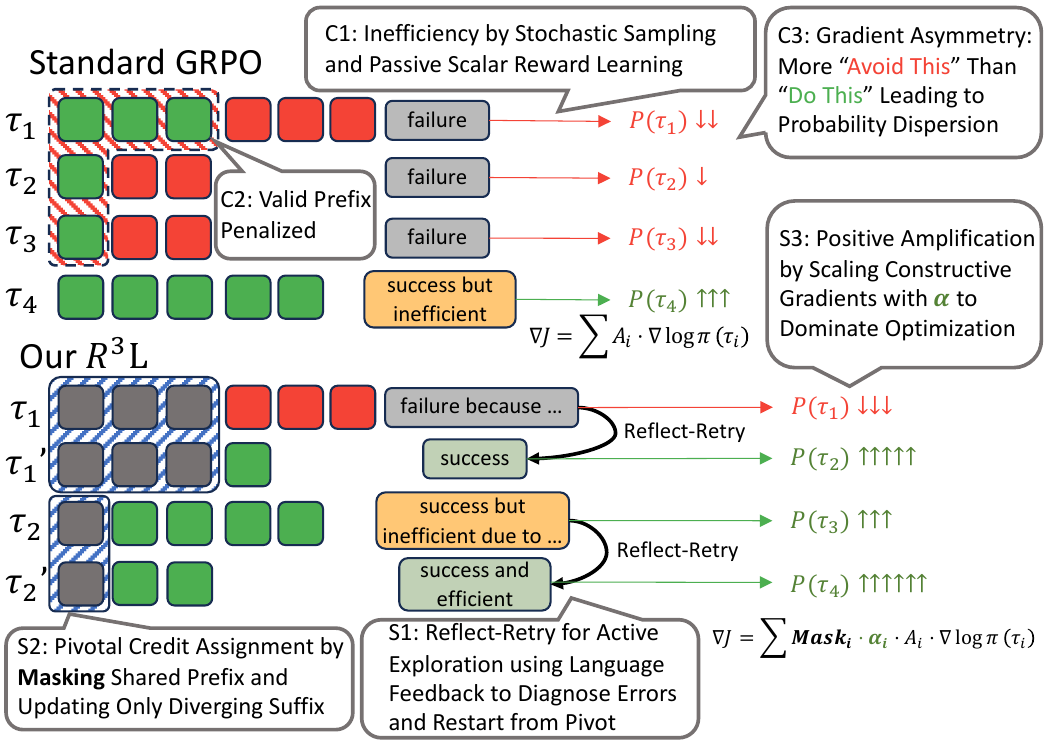}
    \caption{Comparison between standard RL (GRPO) and R$^3$L. \textcolor{red}{Red} blocks indicate erroneous steps, \textcolor{green!50!black}{Green} blocks indicate correct steps, and \textcolor{gray}{Gray} blocks indicate masked prefix excluded from gradient updates. Standard RL suffers from (C1) inefficient stochastic sampling, (C2) valid prefix penalization, and (C3) gradient asymmetry due to failure dominance. R$^3$L addresses these via (S1) reflect-then-retry for active exploration, (S2) pivotal credit, and (S3) positive amplification. The detailed R$^3$L framework is illustrated in Figure~\ref{fig:architecture}.}
    \label{fig:motivation}
\end{wrapfigure}

As depicted in Figure \ref{fig:motivation}, we propose R$^3$L, \textbf{R}eflect-then-\textbf{R}etry \textbf{R}einforcement \textbf{L}earning with Language-Guided Exploration, Pivotal Credit, and Positive Amplification. It synthesizes successful trajectories through language-guided reflection, refines credit assignment through contrastive structure, and stabilizes training through positive advantage amplification. To overcome the low success rates of stochastic sampling, R$^3$L employs a reflect-then-retry mechanism that leverages language feedback to diagnose errors and identify failure points in unsuccessful trajectories, then restarts generation from these pivots with corrective guidance to synthesize successful trajectories at reduced rollout cost. Crucially, the training data removes guidance descriptions so that the model can internalize corrections in the inference stage. To refine credit assignment, we introduce Pivotal Credit Assignment that exploits the contrastive structure between base and retry trajectories. Since they share the same prefix up to the pivot, we exclude this prefix from gradient updates and focus optimization on the diverging suffix where clear signals exist. To ensure stable learning from synthesized off-policy data, we propose Positive Amplification that scales the advantages of successful trajectories. This ensures that positive signals dominate optimization even when failures outnumber successes, preventing entropy collapse. Reflection and retry skills are maintained through auxiliary meta-tasks trained on verified corrections. Our contributions are summarized as follows:
\begin{itemize}
\item We propose a language-guided exploration strategy that synthesizes successful trajectories by diagnosing errors and restarting generation from identified failure points with corrective guidance, improving exploration efficiency while reducing rollout cost.
\item We present Pivotal Credit Assignment and Positive Amplification to stabilize training by focusing gradient updates on diverging suffixes with clear contrastive signals and ensuring positive signals dominate optimization.
\item Extensive experiments demonstrate 5\% to 52\% relative improvements over baselines on agentic tasks, including ALFWorld, WebShop, and ScienceWorld, and mathematical reasoning benchmarks.
\end{itemize}

\section{Related Work}
\label{sec:related_work}
Reinforcement learning trains LLMs on self-generated trajectories to enhance reasoning and agentic capabilities. Typically, methods like GRPO \cite{shao2024deepseekmath} use group-relative normalization to estimate advantages without learned critics, enabling recent advances in mathematical reasoning \cite{guo2025deepseek}. However, reliance on stochastic sampling and sparse rewards introduces challenges in exploration efficiency and training stability.

\subsection{Exploration Efficiency}
Stochastic sampling struggles on difficult tasks where successful trajectories are rare. Sampling-based approaches like DAPO \cite{yu2025dapo} and RAFT \cite{dongraft} compensate through oversampling and filtering, ensuring gradient validity at significant computational cost. Correction-based approaches instead leverage feedback to synthesize improved trajectories directly. HINT \cite{wang2025hint} and Agent-RLVR \cite{da2025agent} employ heuristic guidance and external critics, Goedel-Prover-V2 \cite{lin2025goedel} uses scaffolded synthesis, and Reflect-Retry-Reward \cite{bensal2025reflect} rewards self-reflection tokens that lead to successful retries. While these methods improve trajectory quality, they introduce distributional shifts that can destabilize training if not properly managed \cite{zheng2025prosperity}. R$^3$L combines language-guided synthesis with positive amplification to stabilize learning from off-policy corrections. 

\subsection{Training Stability}
Learning from sampled trajectories introduces gradient variance and credit assignment challenges. To reduce variance, GSPO \cite{zheng2025group} replaces token-level importance weights with sequence-level ratios, while BAPO \cite{xi2025bapo} introduces adaptive clipping to mitigate negative-sample dominance. For credit assignment, trajectory-level rewards penalize valid prefixes when later errors occur. Critique-GRPO \cite{zhang2025critique} uses natural language critiques to guide refinements and applies weighted advantages to the best refinement in each group. Process Reward Models \cite{wang2024math} offer step-level supervision but require expensive annotation and produce unstable signals \cite{xiong2024watch}. Alternative approaches like GiGPO \cite{feng2025group} and VinePPO \cite{kazemnejad2024vineppo} estimate step-level credit through anchor states or Monte Carlo rollouts. R$^3$L takes a distinct way, exploiting the contrastive structure of retry data for precise credit assignment without external annotation, and amplifying positive signals to stabilize off-policy learning.\looseness=-1

Similarly, Reflect-Retry-Reward \cite{bensal2025reflect} and Critique-GRPO \cite{zhang2025critique} leverage language feedback for exploration. R$^3$L differs by identifying specific failure points to reduce rollout cost, using model-driven judgments of suboptimality rather than binary verification, learning from all trajectories rather than only the best refinement or reflection tokens, and applying a simple amplification factor uniformly across the exploration group. Context distillation further removes the guidance dependency, enabling corrections to transfer directly to inference.

\section{Preliminaries}
\subsection{Problem Formulation}
We formulate the agent's interaction as a multi-turn decision process. A trajectory $\tau$ consists of $K$ turns, where each turn $k$ comprises an environment observation $x_k$ and the agent's response $y_k$:
\begin{equation}
    \tau = (x_1, y_1, x_2, y_2, \ldots, x_K, y_K).
\end{equation}
Given a reward function $R(\cdot)$, the reinforcement learning objective is to maximize the expected reward:
\begin{equation}
    J(\theta) = \mathbb{E}_{\tau \sim \pi_{\theta}} [R(\tau)].
\end{equation}
Each response $y_k$ consists of $T_k$ tokens. Let $h_k = (x_1, y_1, \ldots, x_{k-1}, y_{k-1}, x_k)$ denote the history up to turn $k$. The policy gradient is derived as:
\begin{equation}
    \nabla_{\theta} J(\theta) \!\!=\!\! \mathbb{E}_{\tau \sim \pi_{\theta}} \!\!\left[ \sum_{k=1}^{K} \!\! \sum_{t=1}^{T_k} \!\! \nabla_{\theta} \! \log \pi_{\theta}(y_k^t | h_k, y_k^{<t}) \!\cdot\! A_k^t \! \right]\!,
\end{equation}
where $\pi_{\theta}(y_k^t | h_k, y_k^{<t})$ represents the probability of token $y_k^t$ given the history, and $A_k^t$ denotes the advantage function.

\subsection{Group Relative Policy Optimization}
Group Relative Policy Optimization (GRPO) \cite{shao2024deepseekmath} estimates advantages via group-wise reward normalization, eliminating the need for a learned critic. It maintains a trainable policy $\pi_{\theta}$, a behavior policy $\pi_{\theta_{old}}$ for sampling, and a frozen reference policy $\pi_{ref}$ for KL regularization.

For each query, GRPO samples a group of $N$ trajectories $\mathcal{G} = \{\tau_1, \ldots, \tau_N\}$ using $\pi_{\theta_{old}}$. The advantage for a specific trajectory $\tau_i \in \mathcal{G}$ is computed by normalizing its reward against the group statistics:
\begin{equation}
    A(\tau_i) = \frac{R(\tau_i) - \bar{R}}{\sigma_R},
    \label{eq:grpo_advantage}
\end{equation}
where $\bar{R}$ and $\sigma_R$ denote the mean and standard deviation of rewards within group $\mathcal{G}$.
To balance computational efficiency with stability, $\pi_{\theta_{old}}$ is synchronized with $\pi_{\theta}$ every $S$ steps. To manage the resulting off-policy divergence, GRPO employs importance sampling with clipping as follows:
\begin{equation}
    \mathcal{L}_{i,k,t}\!\! =\!\! \min\! \! \left( \! r_{i,k,t} \hat{A}_{i,k,t},\! \text{clip}(r_{i,k,t}, \! 1 \!\!- \!\epsilon, \! 1 \!+\!\epsilon) \hat{A}_{i,k,t} \! \right)\!,
\end{equation}
where $r_{i,k,t} = \frac{\pi_{\theta}(y_k^t | h_k, y_k^{<t})}{\pi_{\theta_{old}}(y_k^t | h_k, y_k^{<t})}$ is the importance sampling ratio for trajectory $\tau_i$. The full GRPO objective incorporates a KL penalty to constrain policy updates:
\begin{equation}
\begin{split}
    \mathcal{J}_{GRPO}(\theta) = &\mathbb{E}_{\mathcal{G} \sim \pi_{\theta_{old}}} \left[ \frac{1}{N} \sum_{i=1}^N \frac{1}{|\tau_i|} \sum_{k,t} \mathcal{L}_{i,k,t} \right] \\
    &- \beta D_{KL}(\pi_{\theta} || \pi_{ref}).
\end{split}
\label{eq:grpo_full}
\end{equation}

\section{Methodology}
\begin{figure*}[t]
\centering
\begin{minipage}[t]{1\linewidth}
\centering
\includegraphics[width=1.0\textwidth]{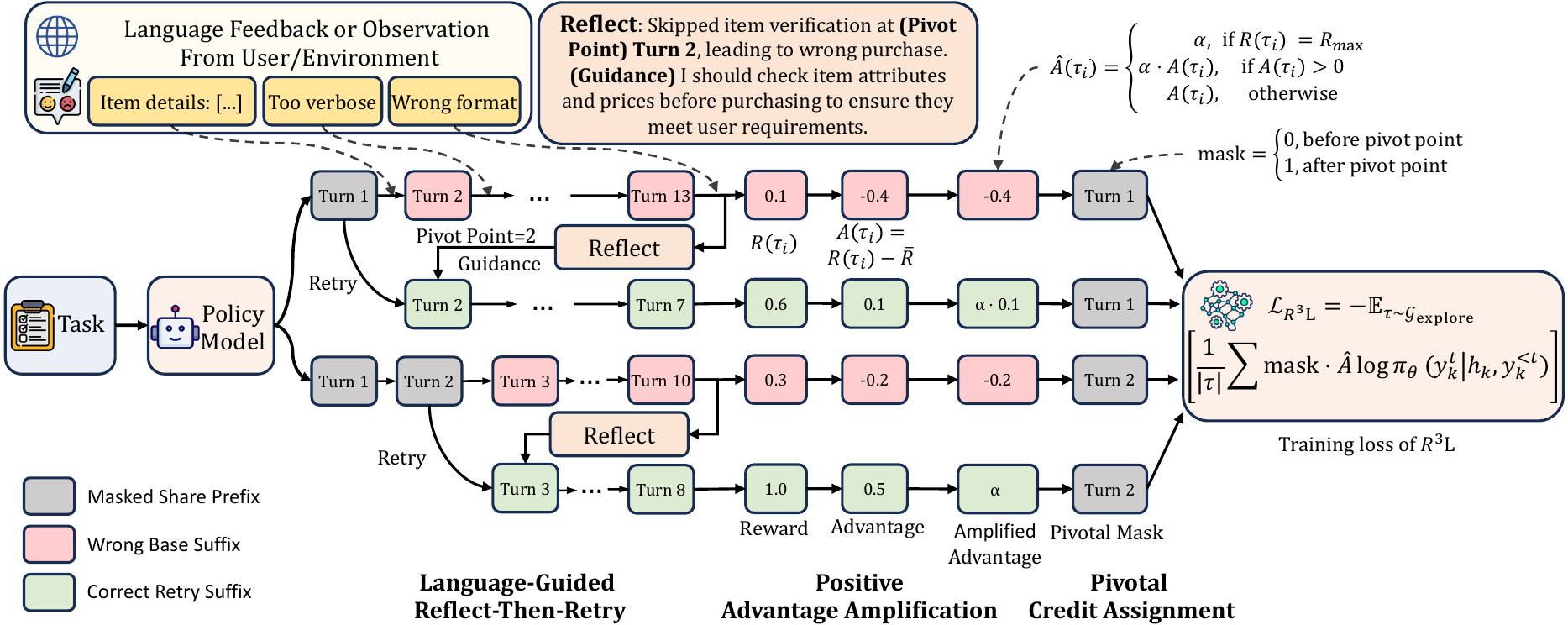}
\end{minipage}
\centering
\caption{Overview of the R$^3$L framework. The method utilizes Language-Guided Reflect-Then-Retry to synthesize high-reward trajectories via verbal feedback. To preserve valid steps, Pivotal Credit Assignment masks shared prefixes to isolate critical decision points, while Positive 
Amplification scales the advantages of successful trajectories to stabilize off-policy training.}
\label{fig:architecture}
\end{figure*}
As illustrated in Figure \ref{fig:architecture}, R$^3$L alternates between trajectory synthesis and policy optimization. In the synthesis phase, a reflect-then-retry mechanism diagnoses failures using language feedback and restarts generation from identified pivot points, transforming failed attempts into successful ones. In the optimization phase, the policy learns from both base and retry trajectories. Pivotal Credit Assignment focuses gradient updates on the diverging suffix where contrastive signals exist, while Positive Amplification upweights successful trajectories to ensure positive signals guide optimization.

\subsection{Language-Guided Reflect-Then-Retry}
The core challenge in RL is discovering successful trajectories to learn from. Stochastic sampling often fails on difficult tasks, and compensating through repeated rollouts is costly. Instead, we leverage language feedback to actively synthesize high-quality trajectories. We construct four types of training data: base trajectories from stochastic sampling, distillation trajectories from guided correction, and two meta-tasks for maintaining reflection and retry skills. The detailed protocols are shown in Appendix \ref{sec:appendix_trajectory}.

\textbf{Base Trajectories.} To present a fair comparison with standard RL methods, we allocate half the sampling budget to standard exploration and the other half for the retried samples, so the total sample size keeps the same. Given a query $x$, $N/2$ base trajectories are sampled from the behavior policy $\pi_{\theta_{old}}$. These establish baseline performance and provide raw material for reflection:
\begin{equation}
    \mathcal{D}_{base} = \{\tau_i\}_{i=1}^{N/2}, \quad \tau_i \sim \pi_{\theta_{old}}(\cdot | x).
\end{equation}

\textbf{Distillation Trajectories.} For each base sample, the model reflects on the trajectory to produce a structured diagnosis, including outcome classification (success, success but inefficient, or failure), root cause analysis, improvement suggestions, and the pivot turn $k_{pivot}$ where the issue first manifested. We refer to the turns preceding the pivot as the \emph{prefix} ($\tau_{<k_{pivot}}$) and the turns from the pivot onward as the \emph{suffix} ($\tau_{\geq k_{pivot}}$); both credit assignment and trajectory stitching operate at this turn level. For trajectories not classified as fully successful, this diagnosis is embedded into a guidance prompt, and generation restarts from $k_{pivot}$ conditioned on this guidance to obtain a corrected suffix. Crucially, we construct training inputs by pairing the original prefix with the corrected suffix, deliberately omitting the guidance. This forces the model to internalize corrections rather than rely on explicit prompts, ensuring they transfer directly to inference:
\begin{equation}
    \mathcal{D}_{distill} = \{(\tau_{i,<k_{pivot}}, \tau'_{i,\geq k_{pivot}})\}_{i=1}^{N/2}.
\end{equation}

\textbf{Learning to Reflect and Retry.} Reflection and retry are learnable skills that require explicit supervision to maintain throughout training, so we devise two auxiliary tasks. We select instances where the retry trajectory achieves a higher reward than the original attempt as verified successful corrections. Specifically, the reflection task trains the model to produce structured diagnoses $r$ given trajectories with environment feedback $f$, while the retry task trains the model to generate corrected suffixes given prefixes concatenated with guidance $g$ (denoted $\oplus$):
\begin{equation}
\begin{gathered}
    \mathcal{D}_{reflect} = \{([\tau_i, f_i], r_i)\}, \\
    \mathcal{D}_{retry} = \{(\tau_{i,<k_{pivot}} \oplus g_i, \tau'_{i,\geq k_{pivot}})\}.
\end{gathered}
\end{equation}

\textbf{Training Groups.} $\mathcal{D}_{base}$ and $\mathcal{D}_{distill}$ form the exploration group $\mathcal{G}_{explore}$ for RL optimization, while $\mathcal{D}_{reflect}$ and $\mathcal{D}_{retry}$ serve as auxiliary SFT objectives to sustain the exploration engine throughout training as Equation \ref{eq:R3L}. Unlike Critique-GRPO which performs $N$ explorations followed by $N$ retries and selects only the best refinement, R$^3$L allocates the budget evenly between base and retry trajectories, learns from all samples in the group, and explicitly maintains reflection and retry skills through auxiliary tasks.

\subsection{Pivotal Credit Assignment}
Standard trajectory-level credit assigns the same reward signal to all tokens, penalizing all tokens when later errors occur. A failure at the final step suppresses the correct reasoning that preceded it, introducing noise into gradient estimates.

To assign credit more precisely, we leverage the contrastive structure of the exploration group. Base and distillation trajectories share an identical prefix up to the pivot turn $k_{pivot}$ and diverge only in the suffix. Since both trajectories behave identically in the prefix region, this shared portion provides no information about which path is better. Only the diverging suffix, where one succeeds and the other does not, reveals which decision was correct.\looseness=-1

We exploit this structure by applying a gradient mask that excludes the shared prefix from updates. For the $t$-th token in turn $k$:
\begin{equation}
    \text{mask}_k^t = \begin{cases}
        0 & \text{if } k < k_{pivot}, \\
        1 & \text{otherwise}.
    \end{cases}
\end{equation}
Tokens before the pivot turn receive zero weight in the gradient, focusing optimization on the critical branching decision. Functionally, the shared prefix acts as a control variate, because policy behavior is identical in this region, including it would introduce variance without providing useful signal. Masking cancels this correlated noise and maximizes the signal-to-noise ratio, as analyzed in Appendix \ref{sec:analysis_pivot}.

\subsection{Positive Amplification}
\label{sec:positive_amplification}
While the reflect-then-retry mechanism synthesizes high-quality trajectories, learning from them poses stability challenges. On difficult tasks, failures often dominate the exploration group. Standard GRPO computes advantages by normalizing rewards against the group mean, so when most trajectories fail, this normalization dilutes the positive signal from the few successful ones, causing them to be overwhelmed by the mass of negative samples. The challenge is further compounded by distribution shifts, as retry trajectories are generated under guidance, creating a mismatch with the current policy that standard importance sampling struggles to correct.\looseness=-1

As analyzed in Appendix \ref{sec:entropy_collapse}, this dilution leads to gradient asymmetry. Failed trajectories provide destructive gradients that suppress incorrect actions, while successful trajectories provide constructive gradients that reinforce correct paths. In difficult tasks, destructive signals are abundant but low-value, merely telling the model what not to do. Constructive signals are rare but high-value. Without intervention, the aggregate gradient of dense negative samples structurally overpowers the rare positive signals, forcing the policy to unlearn behaviors blindly and driving the distribution toward high entropy rather than convergence.

Recent methods address this imbalance through various mechanisms. DeepSeek-V3.2 \cite{liu2025deepseek32} masks out negative-advantage sequences with high policy divergence. BAPO \cite{xi2025bapo} dynamically adjusts clipping bounds through multiple hyperparameters to rebalance positive and negative contributions. Critique-GRPO \cite{zhang2025critique} applies policy shaping only to refinements. Positive Amplification takes a simpler approach: It suffices to ensure that constructive gradients dominate within each group. A single amplification factor $\alpha > 1$ applied uniformly suffices. For trajectory $\tau$ with advantage $A(\tau)$:
\begin{equation}
\label{eq:reweight_advantage}
    \hat{A}(\tau) = \begin{cases}
    \alpha & \text{if } R(\tau) = R_{max}, \\
    \alpha \cdot A(\tau) & \text{if } A(\tau) > 0, \\
    A(\tau) & \text{otherwise}.
    \end{cases}
\end{equation}
Trajectories achieving the maximum reward in the group receive the full amplification factor, other positive-advantage trajectories are scaled proportionally, and negative-advantage trajectories remain unchanged. By scaling up positive advantages, constructive gradients become strong enough to dominate the group, channeling probability mass toward discovered solutions rather than letting it disperse. We find $\alpha=3.0$ works well across tasks, with sensitivity analysis in Appendix \ref{sec:hyperparam}.

Combining the pivotal mask with amplified advantages, the final R$^3$L objective is:
\begin{equation}
\begin{split}
    \mathcal{L}_{R^3L} = & -\mathbb{E}_{\tau \sim \mathcal{G}_{explore}} \bigg[ \frac{1}{|\tau|} \sum_{k=1}^{K} \sum_{t=1}^{T_k} \text{mask}_k^t \cdot \hat{A}(\tau) \\
    & \qquad \cdot \log \pi_{\theta}(y_k^t | h_k, y_k^{<t}) \bigg] \\
    & + \mathcal{L}_{SFT}(\mathcal{D}_{reflect} \cup \mathcal{D}_{retry}),
\end{split}
\label{eq:R3L}
\end{equation}
where $\mathcal{L}_{SFT}$ is the auxiliary supervised fine-tuning loss on verified successful corrections to maintain reflection and retry skills throughout training. Unlike standard GRPO, we remove both importance sampling and KL constraints. Importance sampling becomes unreliable for retry trajectories generated under guidance, as the behavior distribution differs fundamentally from the current policy. KL constraints are unnecessary because positive amplification already prevents the policy from drifting toward high entropy. This simplification reduces memory overhead and computational cost while maintaining training stability, as validated for stability in Appendix~\ref{sec:stability} and for compute in Appendix~\ref{sec:compute}.\looseness=-1

\section{Experiments}
\subsection{Experimental Setup}
\textbf{Benchmarks.} We evaluate R$^3$L on two categories of tasks:
\begin{itemize}[leftmargin=*,nosep]
    \item \textbf{Agentic Environments:} ALFWorld (language-only version) \cite{shridhar2020alfworld} for embodied decision-making, WebShop \cite{yao2022webshop} for web navigation, and ScienceWorld \cite{wang2022scienceworld} for long-horizon scientific reasoning.
    \item \textbf{Mathematical Reasoning:} We train the model on the DAPO training set and evaluate on GSM8K \cite{cobbe2021training}, Math500 \cite{lightman2023let}, MinervaMath \cite{lewkowycz2022solving}, OlympiadBench \cite{gao2024omni}, AMC23 \cite{amc}, and the DAPO test set \cite{yu2025dapo} to assess cross-distribution generalization. Following DeepSeek-R1~\cite{guo2025deepseek}, we use the \texttt{<think>/<answer>} output format, which also matches our agent environment defaults. Post-training scores on distribution-distant benchmarks can therefore fall below zero-shot \texttt{\textbackslash boxed\{\}} numbers from backbone technical reports; see Appendix~\ref{sec:eval_protocol}.
\end{itemize}

\textbf{Baselines.} We compare against rejection sampling methods like RAFT \cite{dongraft}, group-relative policy optimization variants including GRPO \cite{shao2024deepseekmath}, OPMD \cite{yao2025group}, and GSPO \cite{zheng2025group}, as well as language-feedback guided approaches such as Critique-GRPO \cite{zhang2025critique} and Reflect-Retry-Reward \cite{bensal2025reflect}, which we abbreviate as Reflect-GRPO. All baselines are reproduced within the Trinity-RFT framework \cite{pan2025trinity} to ensure fair comparison.

\begin{table*}[t]
\centering
\resizebox{\textwidth}{!}{%
\begin{tabular}{ll ccc cccccc}
\toprule
& & \multicolumn{3}{c}{\textbf{Agentic Environments}} & \multicolumn{6}{c}{\textbf{Mathematical Reasoning}} \\
\cmidrule(lr){3-5} \cmidrule(lr){6-11}
\textbf{Model} & \textbf{Method} & \textbf{ALFWorld} & \textbf{WebShop} & \textbf{ScienceWorld} & \textbf{GSM8K} & \textbf{Math500} & \textbf{MinervaMath} & \textbf{Olympiad} & \textbf{AMC23} & \textbf{DAPO} \\
\midrule
\multirow{7}{*}{\texttt{Qwen2.5-1.5B-Ins}}
& RAFT & 0.826 & 0.450 & 0.110 & 0.434 & 0.204 & 0.051 & 0.053 & 0.125 & 0.086 \\
& OPMD & 0.835 & 0.561 & 0.356 & 0.463 & 0.292 & 0.062 & 0.123 & 0.125 & 0.070 \\
& GRPO & 0.720 & \underline{0.614} & \underline{0.366} & 0.474 & 0.368 & 0.099 & 0.114 & 0.250 & 0.123 \\
& GSPO & 0.857 & 0.566 & 0.200 & 0.518 & 0.196 & 0.077 & 0.087 & 0.275 & \underline{0.142} \\
& Reflect-GRPO & 0.878 & 0.559 & 0.347 & 0.672 & 0.376 & 0.102 & \underline{0.130} & \underline{0.300} & 0.136 \\
& Critique-GRPO & \underline{0.914} & 0.613 & 0.314 & \textbf{0.798} & \underline{0.404} & \underline{0.110} & 0.124 & 0.275 & 0.133 \\
\rowcolor{cyan!10} & \textbf{R$^3$L (Ours)} & \textbf{0.928} & \textbf{0.663} & \textbf{0.385} & \underline{0.721} & \textbf{0.424} & \textbf{0.125} & \textbf{0.151} & \textbf{0.325} & \textbf{0.156} \\
\midrule
\multirow{7}{*}{\texttt{Qwen2.5-7B-Ins}}
& RAFT & 0.914 & 0.682 & 0.201 & \underline{0.884} & \underline{0.592} & \underline{0.257} & 0.210 & 0.575 & 0.383 \\
& OPMD & 0.914 & 0.684 & 0.359 & 0.836 & 0.352 & 0.252 & 0.133 & 0.225 & 0.270 \\
& GRPO & 0.933 & 0.709 & 0.378 & 0.846 & 0.572 & 0.239 & 0.277 & \underline{0.675} & 0.393 \\
& GSPO & 0.895 & 0.720 & 0.363 & 0.866 & 0.548 & 0.209 & \underline{0.284} & \underline{0.675} & \underline{0.413} \\
& Reflect-GRPO & \underline{0.942} & \underline{0.723} & 0.356 & 0.765 & 0.532 & 0.194 & 0.250 & 0.550 & 0.396 \\
& Critique-GRPO & 0.921 & 0.714 & \underline{0.388} & 0.678 & 0.522 & 0.152 & 0.170 & 0.300 & 0.390 \\
\rowcolor{cyan!10} & \textbf{R$^3$L (Ours)} & \textbf{0.948} & \textbf{0.757} & \textbf{0.403} & \textbf{0.897} & \textbf{0.658} & \textbf{0.275} & \textbf{0.301} & \textbf{0.700} & \textbf{0.436} \\
\midrule
\multirow{7}{*}{\texttt{Qwen3-4B}}
& RAFT & 0.886 & 0.659 & 0.350 & 0.914 & 0.638 & 0.326 & 0.516 & 0.650 & 0.505 \\
& OPMD & 0.878 & 0.678 & 0.342 & 0.922 & 0.663 & 0.358 & 0.492 & 0.675 & 0.531 \\
& GRPO & 0.912 & 0.695 & 0.368 & \underline{0.934} & 0.718 & 0.350 & 0.546 & 0.725 & 0.540 \\
& GSPO & 0.916 & 0.707 & 0.376 & 0.928 & \underline{0.722} & 0.364 & 0.538 & \textbf{0.800} & 0.548 \\
& Reflect-GRPO & 0.938 & 0.724 & \underline{0.378} & 0.920 & 0.654 & 0.366 & \underline{0.552} & 0.650 & 0.547 \\
& Critique-GRPO & \underline{0.942} & \underline{0.725} & 0.374 & 0.926 & 0.705 & \underline{0.368} & 0.544 & \underline{0.775} & \underline{0.558} \\
\rowcolor{cyan!10} & \textbf{R$^3$L (Ours)} & \textbf{0.962} & \textbf{0.746} & \textbf{0.398} & \textbf{0.948} & \textbf{0.753} & \textbf{0.383} & \textbf{0.571} & \textbf{0.800} & \textbf{0.583} \\
\bottomrule
\end{tabular}
}
\caption{Main results across agentic environments and mathematical reasoning benchmarks. We report Average Reward for all tasks. \textbf{Bold} indicates the best performance and \underline{underline} indicates the second best.}
\label{tab:main_results}
\vspace{-0.5em}
\end{table*}

\begin{table*}[t]
\centering
\resizebox{\textwidth}{!}{%
\begin{tabular}{lccccccccc}
\toprule
& \multicolumn{3}{c}{\textbf{Agentic Environments}} & \multicolumn{6}{c}{\textbf{Mathematical Reasoning}} \\
\cmidrule(lr){2-4} \cmidrule(lr){5-10}
\textbf{Method} & \textbf{ALFWorld} & \textbf{WebShop} & \textbf{ScienceWorld} & \textbf{GSM8K} & \textbf{Math500} & \textbf{MinervaMath} & \textbf{Olympiad} & \textbf{AMC23} & \textbf{DAPO} \\
\midrule
\rowcolor{cyan!10}\textbf{R$^3$L (Full)} & \textbf{0.928} & \textbf{0.663} & \textbf{0.385} & \textbf{0.721} & \textbf{0.424} & \textbf{0.125} & \textbf{0.151} & \textbf{0.325} & \textbf{0.156} \\
w/o Positive & 0.881 & 0.646 & 0.373 & 0.685 & 0.391 & 0.112 & 0.143 & 0.300 & 0.144 \\
w/o Credit & 0.914 & 0.649 & 0.378 & 0.706 & 0.410 & 0.118 & 0.150 & 0.325 & 0.153 \\
w/o Reflect & 0.894 & 0.628 & 0.371 & 0.562 & 0.389 & 0.105 & 0.133 & 0.275 & 0.141 \\
\midrule
GRPO (Baseline) & 0.720 & 0.614 & 0.366 & 0.474 & 0.368 & 0.099 & 0.114 & 0.250 & 0.123 \\
\bottomrule
\end{tabular}
}
\caption{Ablation study on Qwen2.5-1.5B-Instruct. Each row removes the component from the R$^3$L framework. Removing Reflect-then-Retry inherently disables Pivotal Credit Assignment as the Credit module depends on pivot points identified by reflection.}
\label{tab:ablation}
\vspace{-1em}
\end{table*}

\textbf{Implementation.} We conduct main experiments on Qwen2.5-1.5B-Instruct, Qwen2.5-7B-Instruct, and Qwen3-4B, with cross-architecture evaluation on Llama-3.2-3B-Instruct reported in Appendix~\ref{sec:cross_arch}. Key hyperparameters include group size $N=8$, amplification factor $\alpha=3.0$, and synchronization step $S=1$ for updating the behavior policy. We report trajectory-level \textbf{Average Reward} as the primary metric, measuring task completion rate for agentic environments and solution accuracy for mathematical reasoning. Subsequent ablation studies and analyses are mainly performed using Qwen2.5-1.5B-Instruct unless otherwise specified. More baseline reproduction and task-specific implementation details are provided in Appendix \ref{app:implementation}.

\subsection{Main Results}
Table \ref{tab:main_results} presents results across agentic environments and mathematical reasoning benchmarks. Rejection sampling methods like RAFT reach 0.914 on ALFWorld and 0.884 on GSM8K with the 7B model, but degrade to 0.826 and 0.434 with the 1.5B model, where successful samples are scarce. Group-relative methods provide solid baselines but rely on stochastic sampling. GRPO reaches 0.720 on ALFWorld and 0.474 on GSM8K with the 1.5B model, while GSPO improves to 0.857 on ALFWorld through variance reduction yet lacks active trajectory synthesis. Language-feedback guided methods improve exploration. Reflect-GRPO reaches 0.878 on ALFWorld and 0.672 on GSM8K with the 1.5B model, but cannot stabilize learning from the shifted distribution. Critique-GRPO further improves to 0.914 on ALFWorld and 0.798 on GSM8K by selecting the best refinement, though discarding other refinement trajectories limits signal diversity.

R$^3$L achieves the best or second-best performance on all 27 settings. On agentic environments, R$^3$L ranks first on all 9 settings, reaching 0.928 on ALFWorld, 0.663 on WebShop, and 0.385 on ScienceWorld with the 1.5B model. The gains are most pronounced on WebShop and ScienceWorld, surpassing the strongest baselines by 8.0\% and 5.2\% respectively, where Pivotal Credit Assignment protects valid prefixes from late-error penalties. On mathematical reasoning, R$^3$L achieves 0.897 on GSM8K, 0.658 on Math500, and 0.301 on OlympiadBench with the 7B model. The only exception to achieving the first best is GSM8K on 1.5B, where Critique-GRPO outperforms R$^3$L with 0.798 versus 0.721. The reason for the exception is that GSM8K is simple enough that stochastic sampling already produces sufficient successful trajectories. Nevertheless, R$^3$L still achieves 52\% gains over GRPO on GSM8K, and harder benchmarks benefit substantially more from its deliberate mechanisms for trajectory synthesis, fine-grained credit, and stable learning from sparse positive signals.

Results on the modern Qwen3-4B backbone further confirm that R$^3$L's benefits persist on stronger pretrained models. Even with strong baselines on Qwen3-4B, where GRPO reaches 0.934 on GSM8K and GSPO reaches 0.722 on Math500, R$^3$L still pushes these to 0.948 and 0.753, respectively, attaining the best or tied-best score on all 9 benchmarks. This suggests that the structural bottlenecks R$^3$L targets, including prefix penalty, entropy collapse, and inefficient exploration, persist across model generations rather than being resolved by scale or newer pretraining alone.

\subsection{Ablation Study} 
\vspace{-0.3em}
To validate the necessity of each component, we systematically remove Positive Amplification, Pivotal Credit Assignment, and Reflect-then-Retry. Since the Credit module depends on pivot points identified by reflection, removing the Reflect module inherently disables the Credit mechanism.

Table \ref{tab:ablation} presents the ablation results. Removing Reflect-then-Retry causes the most severe degradation, with GSM8K dropping from 0.721 to 0.562 and WebShop from 0.663 to 0.628. This confirms that stochastic sampling alone cannot penetrate sparse-reward landscapes, and active trajectory synthesis is essential for discovering successful solutions. Excluding Positive Amplification degrades performance across all benchmarks, with ALFWorld dropping from 0.928 to 0.881 and GSM8K from 0.721 to 0.685. In failure-dominated regimes, the abundance of negative samples overwhelms rare positive signals, and without amplification, gradients focus on suppressing errors rather than reinforcing correct paths. Omitting Pivotal Credit Assignment has the smallest but consistent impact, with ALFWorld dropping from 0.928 to 0.914 and GSM8K from 0.721 to 0.706. The effect is most visible on long-horizon tasks where trajectory-level rewards penalize valid prefixes when errors occur.

\subsection{Efficacy of Language-Guided Reflection}
\label{sec:reflect_efficacy}
\begin{wrapfigure}{r}{0.55\textwidth}
    \centering
    \begin{minipage}{0.49\linewidth}
        \centering
        \begin{tikzpicture}
            \begin{axis}[
                width=\linewidth, height=3.8cm,
                xlabel={Training Steps}, ylabel={Reward Gain},
                grid=major, ymin=0,
                font=\scriptsize, yticklabel={\pgfmathprintnumber{\tick}},
            ]
            \addplot[color=blue, mark=square*, mark size=0pt, thick, smooth, tension=1.0]
            table[x=step, y=gain] {figure/src/alfworld-gain.dat};
            \addplot[color=red, mark=*, mark size=0pt, thick, smooth, tension=1.0]
            table[x=step, y=gain] {figure/src/alfworld-7b-gain.dat};
            \addplot[color=green!70!black, mark=triangle*, mark size=0pt, thick, smooth, tension=1.0]
            table[x=step, y=gain] {figure/src/webshop-gain.dat};
            \addplot[color=orange, mark=diamond*, mark size=0pt, thick, smooth, tension=1.0]
            table[x=step, y=gain] {figure/src/dapo-gain.dat};
            \end{axis}
        \end{tikzpicture}
        \centerline{\scriptsize (a) Avg. Reward Gain}
    \end{minipage}\hfill
    \begin{minipage}{0.49\linewidth}
        \centering
        \begin{tikzpicture}
            \begin{axis}[
                width=\linewidth, height=3.8cm,
                xlabel={Training Steps}, ylabel={Improv. Rate (\%)},
                grid=major, ymin=0,
                font=\scriptsize, yticklabel={\pgfmathprintnumber{\tick}},
            ]
            \addplot[color=blue, mark=square*, mark size=0pt, thick, smooth, tension=1.0]
            table[x=step, y=rir] {figure/src/alfworld-rir.dat};
            \addplot[color=red, mark=*, mark size=0pt, thick, smooth, tension=1.0]
            table[x=step, y=rir] {figure/src/alfworld-7b-rir.dat};
            \addplot[color=green!70!black, mark=triangle*, mark size=0pt, thick, smooth, tension=1.0]
            table[x=step, y=rir] {figure/src/webshop-rir.dat};
            \addplot[color=orange, mark=diamond*, mark size=0pt, thick, smooth, tension=1.0]
            table[x=step, y=rir] {figure/src/dapo-rir.dat};
            \end{axis}
        \end{tikzpicture}
        \centerline{\scriptsize (b) Retry Improvement Rate}
    \end{minipage}

    \vspace{0.4em}
    {\scriptsize
    \textcolor{blue}{\rule[0.35ex]{1.1em}{1.3pt}}\,ALFWorld (1.5B)\quad
    \textcolor{red}{\rule[0.35ex]{1.1em}{1.3pt}}\,ALFWorld (7B)\par
    \textcolor{green!70!black}{\rule[0.35ex]{1.1em}{1.3pt}}\,WebShop (7B)\quad
    \textcolor{orange}{\rule[0.35ex]{1.1em}{1.3pt}}\,DAPO (7B)\par}

    \caption{Evolution of exploration metrics across environments by the Reflect-then-Retry mechanism. \textbf{(a)} The average Reward Gain of retry trajectory relative to the base trajectory. \textbf{(b)} The percentage of retry trajectories that successfully improved upon the base attempt.}
    \label{fig:exploration_metrics}
\end{wrapfigure}
\begin{table}[t]
\centering
\small
\adjustbox{max width=\linewidth}{%
\begin{tabular}{lcccc}
\toprule
\textbf{Model} & \textbf{ALFWorld} & \textbf{WebShop} & \textbf{ScienceWorld} & \textbf{DAPO} \\
\midrule
Qwen2.5-1.5B-Ins & 64.7\% & 23.7\% & 13.5\% & 16.4\% \\
Llama3.2-3B-Ins & 12.8\% & 31.6\% & 9.2\% & 10.3\% \\
Qwen2.5-7B-Ins & 73.9\% & 36.5\% & 19.6\% & 27.1\% \\
\bottomrule
\end{tabular}%
}
\caption{Average reward improvement rate achieved by the Reflect-then-Retry mechanism.}
\label{tab:retry_gain_stats}
\vspace{-1.2em}
\end{table}

A key question is whether language-guided reflection can reliably improve upon base trajectories throughout training. To analyze this, we track two metrics: the Retry Improvement Rate, defined as the percentage of retry trajectories that achieve higher reward than their corresponding base attempts ($R_{retry} > R_{base}$), and the Reward Gain, defined as the relative improvement $(\bar{R}_{retry} - \bar{R}_{base}) / \bar{R}_{base}$ in average reward. Figure \ref{fig:exploration_metrics} illustrates the temporal evolution of these metrics across environments and model scales. We observe two distinct patterns in the training dynamics.

\textbf{Cold Start and Warm-up.} Most configurations start with Reward Gain near zero, including ALFWorld with the 1.5B model, WebShop with the 7B model, and DAPO with the 7B model. During this phase the model learns to generate structured diagnoses and apply corrections effectively. ALFWorld with the 7B model is the exception, starting with Reward Gain around 0.4 because its stronger base capabilities enable effective reflection without additional adaptation.

\textbf{Sustained Improvement Gap.} After warm-up, all configurations show sustained improvement that persists throughout training. We achieve the highest Reward Gain on ALFWorld, reaching 0.6 with the 1.5B model and 0.8 with the 7B model; DAPO follows with 0.3 and WebShop reaches 0.15. Table \ref{tab:retry_gain_stats} summarizes the average Retry Improvement Rate across the full training process: Qwen2.5-7B-Instruct reaches 73.9\% on ALFWorld and 36.5\% on WebShop, while Qwen2.5-1.5B-Instruct follows with 64.7\% and 23.7\% respectively. Task-level variation in retry gains reflects differences in how amenable each error type is to language-guided correction (Appendix~\ref{sec:error_amenability}).
Two key factors drive this improvement. First, as the model masters environment dynamics, both base performance and reflection quality improve in tandem. Second, the auxiliary SFT on verified corrections explicitly maintains reflection and retry capabilities, preventing the skill degradation observed in methods like Critique-GRPO and Reflect-GRPO that lack such supervision. Together, these factors keep Positive Amplification supplied with high-quality trajectories. Pivot localization accuracy against oracle labels supports this trend (Appendix~\ref{sec:pivot_accuracy}).

\section{Conclusion}
Current reinforcement learning approaches for LLMs are severely constrained by inefficient exploration, coarse credit assignment, and training instability in failure-dominated regimes. To overcome these structural bottlenecks, we introduced R$^3$L, which transforms sparse-reward environments into rich learning opportunities. By leveraging Reflect-then-Retry to actively synthesize high-quality trajectories, employing Pivotal Credit Assignment to protect valid reasoning prefixes, and utilizing Positive Amplification to ensure that successful signals dominate gradient updates, R$^3$L achieves robust learning where standard methods collapse or converge prematurely.

\section*{Limitations}
Despite the merits, we acknowledge that R$^3$L has several limitations as follows. First, the reflection step requires an additional inference pass to diagnose errors. However, the pivot mechanism compensates by restarting generation from identified failure points rather than from scratch, reducing per-step rollout turns by 28--34\% on multi-step agentic tasks. Across four model-task settings we measure, R$^3$L's per-step rollout time stays below GRPO's while training tokens on multi-step tasks are roughly halved, and Critique-GRPO runs 1.8--2.0$\times$ slower due to its $2N$ rollouts (Appendix~\ref{sec:compute}).

Second, we observe a cold-start phenomenon in smaller models. Unlike the 7B model that can immediately leverage reflection prompts, the 1.5B model initially struggles to generate actionable self-corrections. This necessitates a longer warm-up period to bootstrap reflection capability, potentially complicating the training pipeline for low-resource settings.\looseness=-1

Finally, while our framework is theoretically applicable to any domain where a preference signal exists, our experimental scope is limited to tasks with verifiable ground truth. We have not validated R$^3$L in open-ended domains with subjective criteria such as creative writing, where the reliability of automated retry validation remains an open question for future research.

\bibliographystyle{colm2026_conference}
\bibliography{custom}

\appendix

\section{Stability Analysis and Theoretical Foundations}
\label{sec:stability}
\begin{figure*}[t]
    \centering
    \includegraphics[width=1.0\linewidth]{figure/src/training_analysis_final.png}
    \caption{Training Dynamics and Stability Analysis on ALFWorld. \textbf{(a)} Average reward comparison. \textbf{(b)} Reference KL divergence between current policy and reference policy. \textbf{(c)} Gradient norm during training. \textbf{(d)} Policy loss comparison. \textbf{(e)} GRPO update KL between new and old policy. \textbf{(f)} GRPO clip fraction.}
    \label{fig:training_analysis}
\end{figure*}

\subsection{Training Dynamics and Convergence}
\label{sec:training_dynamics}
Reinforcement learning for LLMs is notoriously sensitive to initialization and hyperparameters. To understand the underlying causes of instability, we analyze the training dynamics on ALFWorld as illustrated in Figure \ref{fig:training_analysis}.

\textbf{Performance and Convergence.} As shown in Figure \ref{fig:training_analysis}(a), standard GRPO exhibits earlier initial gains, showing reward improvements around step 100. However, it suffers from premature stagnation, plateauing at a reward of approximately 0.4 with visible variance. In contrast, R$^3$L undergoes a warm-up phase due to initial adaptation to the reflection mechanism. Following step 120, R$^3$L demonstrates a rapid phase transition with a significantly steeper learning curve, quickly surpassing the baseline and achieving superior asymptotic performance of approximately 0.9 without the instability observed in GRPO.

\textbf{Policy Drift and Gradient Anomalies.} The instability of GRPO is further evidenced by policy drift and gradient anomalies. Figure \ref{fig:training_analysis}(b) shows that the Reference KL divergence for GRPO explodes to values exceeding 10.0 after step 200, indicating severe policy drift where the model loses its semantic grounding. This drift forces aggressive clipping of updates, with Clip Fraction surging to 30\% as shown in Figure \ref{fig:training_analysis}(f). Furthermore, Figure \ref{fig:training_analysis}(c) captures a massive gradient spike with norm exceeding 400 around step 190, coinciding with the onset of collapse. R$^3$L eliminates these artifacts entirely, maintaining smooth gradient norms and consistently low KL divergence throughout training. We attribute this stability to the synergy between Pivotal Credit Assignment and Positive Amplification, where the former reduces gradient variance, and the latter rectifies gradient direction. Figure \ref{fig:training_analysis}(d) confirms that R$^3$L maintains positive policy loss throughout training, indicating sustained constructive learning signals.

\subsection{Theoretical Foundations}
\label{sec:theory}
We now provide formal theoretical justifications for the empirical stability observed above.

\subsubsection{Gradient Decomposition and Entropy Collapse}
\label{sec:entropy_collapse}
Standard GRPO employs variance-based normalization $A_i = (R_i - \bar{R})/\sigma$. In failure-dominated regimes, the instability arises not merely from the quantity of negative samples but from the distributional asymmetry of the policy's probability mass combined with gradient conflict.

\begin{definition}[Gradient Decomposition]
The expected policy gradient decomposes into constructive and destructive components:
\begin{equation}
\mathbb{E}[\nabla J(\theta)] = \underbrace{p \cdot \bar{A}^+ \cdot \nabla^+}_{\mathcal{C}: \text{Constructive}} + \underbrace{(1-p) \cdot \bar{A}^- \cdot \nabla^-}_{\mathcal{D}: \text{Destructive}} ,
\end{equation}
where $p$ is the fraction of positive-advantage trajectories, $\bar{A}^+$ and $\bar{A}^-$ are the mean advantages, and $\nabla^+$ and $\nabla^-$ denote the average gradient directions.
\end{definition}

Consider the critical exploration phase where the model has not yet solidified the correct reasoning path but is confident in erroneous heuristics. In this regime, the sampled group $\mathcal{G}$ is structurally dominated by negative samples with $A_{neg} < 0$ that simultaneously hold high probability mass where $\pi_\theta(y_{neg}) \gg 0$. Conversely, the correct solution $y_{pos}$ lies in the low-probability tail with $\pi_\theta(y_{pos}) \approx 0$ and is rarely sampled.

\begin{proposition}[Entropy Collapse]
\label{prop:entropy_collapse}
When $|\mathcal{D}| > |\mathcal{C}|$, the policy gradient primarily decreases the probability of high-probability erroneous tokens. The gradient update satisfies:
\begin{equation}
\Delta\pi(y_{neg}) \!\! < \!\! 0 \!\! \implies \!\! \sum_{v \neq y_{neg}} \!\! \Delta\pi(v) \!\! = \!\! -\Delta\pi(y_{neg}) \!\! > \!\! 0 .
\end{equation}
Without a strong positive attractor, this redistributed mass disperses across the vocabulary, monotonically increasing entropy $H(\pi_\theta)$.
\end{proposition}

\begin{proof}
The probability simplex constraint requires $\sum_v \pi(v) = 1$. When the gradient decreases $\pi(y_{neg})$ by $\delta > 0$, the freed probability mass must be redistributed. Let $z_v$ denote logits. The gradient update on $z_{y_{neg}}$ is strongly negative, proportional to $|\mathcal{D}|$, while updates on other $z_v$ are weakly positive. Without amplification, the positive updates are insufficient to create a sharp mode, leading to a flatter distribution and increased entropy. The model effectively unlearns its structured behavior, manifesting as the reward stagnation in Figure \ref{fig:training_analysis}(a) and policy drift in Figure \ref{fig:training_analysis}(b).
\end{proof}

\subsubsection{Gradient Dominance Condition}
\label{sec:gradient_dominance}

Positive Amplification resolves entropy collapse by enforcing a gradient dominance condition that ensures constructive signals outweigh destructive ones.\looseness=-1

\begin{theorem}[Gradient Dominance Condition]
\label{thm:gradient_dominance}
For the amplified constructive gradient to dominate the destructive gradient, the amplification factor $\alpha$ must satisfy:
\begin{equation}
\alpha > \alpha_{\min} = \frac{(1-p)|\bar{A}^-|}{p \cdot \bar{A}^+}.
\end{equation}
\end{theorem}

\begin{proof}
With amplification factor $\alpha$ applied to positive-advantage trajectories, the expected gradient becomes:
\begin{equation}
\mathbb{E}[\nabla J_\alpha(\theta)] = \alpha \cdot p \cdot \bar{A}^+ \cdot \nabla^+ + (1-p) \cdot \bar{A}^- \cdot \nabla^-.
\end{equation}
For constructive signals to dominate, we require the magnitude of the positive term to exceed the negative term. Assuming comparable gradient norms $\|\nabla^+\| \approx \|\nabla^-\|$, this reduces to $\alpha \cdot p \cdot \bar{A}^+ > (1-p) \cdot |\bar{A}^-|$, yielding the stated bound.
\end{proof}

\begin{corollary}[Robustness of $\alpha = 3.0$]
The choice $\alpha = 3.0$ satisfies the gradient dominance condition across typical reinforcement learning scenarios. In practice, the success fraction $p$ ranges from 0.25 on difficult tasks with weaker models to 0.45 on easier tasks with stronger models. The advantage ratio $|\bar{A}^-|/\bar{A}^+$ stays between 1.0 and 2.0 due to group normalization. Under these conditions, $\alpha_{min}$ ranges from approximately 1.2 to 3.0. Our choice of $\alpha = 3.0$ thus covers the most practical spectrum of typical scenarios, from easy tasks where lower amplification would suffice to difficult tasks where $\alpha = 3.0$ is necessary. Setting $\alpha$ higher risks overfitting to specific retry trajectories, as validated in Appendix \ref{sec:hyperparam}.
\end{corollary}

By satisfying the gradient dominance condition, Positive Amplification constructs a synthetic attractor that channels the probability mass released from suppressing $y_{neg}$ toward $y_{retry}$ rather than scattering it across the vocabulary. This ensures that the primary learning signal is to reinforce correct behavior rather than merely suppress errors, aligning with theoretical insights from RED \cite{yao2025group} that weighting high-reward samples is essential for consistent off-policy convergence.

\subsubsection{Variance Reduction via Pivotal Credit Assignment}
\label{sec:analysis_pivot}

While Positive Amplification addresses gradient direction, Pivotal Credit Assignment reduces gradient variance by exploiting the contrastive structure of base-retry pairs.

\begin{theorem}[Variance Reduction]
\label{thm:variance_reduction}
Let $T$ be the total trajectory length and $T_{pivot}$ the pivot position. Pivotal Credit Assignment reduces gradient variance by:
\begin{equation}
\frac{\mathrm{Var}(\nabla_{\mathrm{R^3L}})}{\mathrm{Var}(\nabla_{\mathrm{GRPO}})} \leq \frac{T - T_{pivot}}{T}.
\end{equation}
\end{theorem}

\begin{proof}
For standard GRPO, the gradient variance is the sum of per-step variances:
\begin{equation}
\mathrm{Var}(\nabla_{\text{GRPO}}) = \sum_{t=1}^{T} \mathrm{Var}(\nabla_t) = T \cdot \sigma^2,
\end{equation}
where $\sigma^2$ is the average per-step variance. With pivotal masking, tokens before $t_{pivot}$ receive zero gradient weight:
\begin{equation}
\mathrm{Var}(\nabla_{\mathrm{R^3L}}) =\sum_{t=t_{pivot}}^{T} \mathrm{Var}(\nabla_t) = (T - T_{pivot}) \cdot \sigma^2.
\end{equation}
Taking the ratio yields the stated bound.
\end{proof}

\begin{proposition}[Covariance Analysis]
\label{prop:covariance}
Let $\hat{g}_{base}$ and $\hat{g}_{retry}$ be gradient estimators from base and retry trajectories. Because they share the prefix, their covariance satisfies:
\begin{equation}
\mathrm{Cov}(\hat{g}_{base}, \hat{g}_{retry}) = \frac{T_{pivot}}{T} \cdot \mathrm{Var}(\hat{g}).
\end{equation}
The variance of the difference estimator is:
\begin{equation}
\mathrm{Var}(\hat{g}_{retry} - \hat{g}_{base}) = 2\left(1 - \frac{T_{pivot}}{T}\right) \cdot \mathrm{Var}(\hat{g}).
\end{equation}
This high correlation significantly reduces effective learning signal variance.
\end{proposition}

Theorem \ref{thm:variance_reduction} reveals a natural coupling between model improvement and variance reduction. As training progresses and the model strengthens, errors occur later in trajectories, pushing pivot points rightward as observed in Figure \ref{fig:pivot_evolution}. By Theorem \ref{thm:variance_reduction}, this rightward shift directly increases variance reduction. Early in training, when pivot points cluster near the beginning, variance reduction is modest. As the model matures and pivot points migrate toward trajectory ends, variance reduction approaches its maximum. For instance, when the average pivot point shifts from $T_{pivot}/T = 0.2$ to $T_{pivot}/T = 0.6$, variance reduction improves from 80\% to 40\% of baseline variance. This self-reinforcing dynamic means that Pivotal Credit Assignment becomes increasingly effective precisely when the model needs fine-grained credit assignment most.

\subsubsection{Off-Policy Stability via Context Distillation}
\label{sec:context_distillation}

A critical question is why R$^3$L remains stable despite learning from off-policy retry data without importance sampling or KL constraints.

\begin{theorem}[Off-Policy Stability]
\label{thm:offpolicy_stability}
Unlike standard off-policy RL requiring importance sampling, R$^3$L's verified distillation is stable without importance weights because learning signals come from verified successful trajectories with $R(\tau') > R(\tau_{base})$, ensuring high-quality supervision regardless of distributional shift. Positive Amplification ensures these signals dominate the gradient, making importance corrections unnecessary. Context distillation removes guidance dependency, aligning training with inference distribution.
\end{theorem}

R$^3$L can be viewed as Filtered Behavioral Cloning, where the filter is reward verification. Unlike OPMD, which imitates stale, potentially noisy samples from the old policy, R$^3$L distills verified high-reward trajectories. By amplifying these golden signals via Positive Amplification, R$^3$L effectively treats off-policy learning as stable supervised distillation rather than fragile importance sampling, preventing the policy drift shown in Figure \ref{fig:training_analysis}(b).

\subsubsection{Convergence Guarantee}
\label{sec:convergence}

\begin{theorem}[Local Convergence]
\label{thm:convergence}
Under the following conditions, the policy sequence $\{\pi_k\}$ generated by R$^3$L converges to a local optimum $\pi^*$: retry success rate $p_{retry} > 0$, amplification factor $\alpha > \alpha_{min}$, pivot identification accuracy exceeds random baseline, and learning rate $\eta_k$ satisfies $\sum_k \eta_k = \infty$ and $\sum_k \eta_k^2 < \infty$.
\end{theorem}

\begin{proof}[Proof Sketch]
Define Lyapunov function $V(\pi) = -\mathbb{E}_{\tau \sim \pi}[R(\tau)]$. Under the stated conditions, each R$^3$L update produces a gradient $g$ such that $\langle g, \nabla V \rangle < 0$ in expectation. The learning rate condition ensures convergence via standard stochastic approximation arguments.
\end{proof}

\begin{proposition}[Sustained Teacher-Student Gap]
\label{prop:gap}
Define Retry Improvement Rate $\text{RIR}_k = \mathbb{P}(R(\tau_{retry}) > R(\tau_{base}) | \pi_k)$. Under R$^3$L training dynamics with auxiliary SFT tasks maintaining reflection quality:
\begin{equation}
\exists \epsilon > 0 : \text{RIR}_k > \epsilon, \quad \forall k.
\end{equation}
This ensures continued supervision even as the model improves, as empirically validated in Section \ref{sec:reflect_efficacy}.
\end{proposition}

\section{Cross-Architecture Evaluation}
\label{sec:cross_arch}

To probe whether R$^3$L's improvements generalize beyond the Qwen family, we additionally evaluate all methods on Llama-3.2-3B-Instruct under the same protocol as the main experiments. Llama-3.2-3B exhibits substantially more volatile RL dynamics than Qwen backbones, a phenomenon also reported in recent studies~\cite{zhang2025interplay} and often attributed to differences in reasoning-oriented mid-training. This manifests as sampling-based methods collapsing on agentic tasks: for example, RAFT drops from 0.826 on Qwen2.5-1.5B to 0.064 on Llama-3.2-3B for ALFWorld, and baseline orderings shift erratically across benchmarks.

\begin{table*}[t]
\centering
\resizebox{\textwidth}{!}{%
\begin{tabular}{l ccc cccccc}
\toprule
& \multicolumn{3}{c}{\textbf{Agentic Environments}} & \multicolumn{6}{c}{\textbf{Mathematical Reasoning}} \\
\cmidrule(lr){2-4} \cmidrule(lr){5-10}
\textbf{Method} & \textbf{ALFWorld} & \textbf{WebShop} & \textbf{ScienceWorld} & \textbf{GSM8K} & \textbf{Math500} & \textbf{MinervaMath} & \textbf{Olympiad} & \textbf{AMC23} & \textbf{DAPO} \\
\midrule
RAFT & 0.064 & 0.434 & 0.074 & 0.620 & 0.336 & 0.095 & 0.096 & 0.275 & 0.156 \\
OPMD & 0.385 & 0.492 & \underline{0.117} & 0.548 & 0.310 & 0.084 & 0.094 & 0.125 & 0.186 \\
GRPO & \textbf{0.521} & 0.520 & 0.076 & 0.664 & 0.398 & 0.128 & \textbf{0.121} & \underline{0.325} & \underline{0.213} \\
GSPO & 0.285 & 0.497 & 0.082 & 0.574 & 0.308 & 0.110 & 0.106 & 0.225 & 0.163 \\
Reflect-GRPO & 0.321 & 0.536 & 0.072 & \underline{0.669} & 0.392 & 0.128 & 0.112 & \underline{0.325} & 0.193 \\
Critique-GRPO & 0.321 & \underline{0.549} & 0.071 & 0.662 & \textbf{0.426} & \underline{0.137} & 0.116 & \textbf{0.375} & 0.208 \\
\rowcolor{cyan!10} \textbf{R$^3$L (Ours)} & \underline{0.495} & \textbf{0.569} & \textbf{0.123} & \textbf{0.688} & \underline{0.408} & \textbf{0.147} & \underline{0.118} & \underline{0.325} & \textbf{0.216} \\
\bottomrule
\end{tabular}
}
\caption{Results on Llama-3.2-3B-Instruct across agentic environments and mathematical reasoning benchmarks. We report Average Reward for all tasks. \textbf{Bold} indicates the best performance and \underline{underline} indicates the second best.}
\label{tab:llama_results}
\vspace{-0.5em}
\end{table*}

Table~\ref{tab:llama_results} reports the complete results. Despite this backbone-level volatility, R$^3$L still ranks first on 5 of 9 benchmarks (WebShop, ScienceWorld, GSM8K, MinervaMath, and DAPO) and second on the remaining 4, staying within 0.05 of the top method in every column. This indicates that the three core components of R$^3$L, namely Reflect-then-Retry, Pivotal Credit Assignment, and Positive Amplification, remain effective when the base policy is less amenable to RL, though absolute scores are bounded by the backbone's RL susceptibility.

\section{Hyperparameter Analysis}
\label{sec:hyperparam}

\subsection{Amplification Factor}
The amplification factor $\alpha$ in Positive Amplification controls the relative weight of successful trajectories in the gradient update. As defined in Equation~\ref{eq:reweight_advantage}, trajectories achieving the maximum reward in the group receive advantage $\alpha$, other positive-advantage trajectories are scaled by $\alpha$, and negative-advantage trajectories remain unchanged. To determine the optimal balance, we evaluate R$^3$L across $\alpha \in \{1.0, 2.0, 3.0, 5.0, 7.0\}$ on GSM8K and WebShop.

\begin{table}[h]
\centering
\adjustbox{max width=\linewidth}{%
\begin{tabular}{lccccc}
\toprule
\textbf{Task} & $\alpha=1.0$ & $\alpha=2.0$ & $\alpha=3.0$ & $\alpha=5.0$ & $\alpha=7.0$ \\
\midrule
ALFWorld & 0.921 & 0.914 & 0.928 & 0.921 & \textbf{0.942} \\
WebShop & 0.647 & 0.650 & \textbf{0.663} & 0.638 & 0.614 \\
ScienceWorld & 0.375 & 0.382 & \textbf{0.385} & 0.366 & 0.353 \\
GSM8K & \textbf{0.732} & 0.718 & 0.721 & 0.671 & 0.717 \\
Math500 & 0.414 & 0.420 & \textbf{0.424} & 0.372 & 0.404 \\
Olympiad & 0.143 & \textbf{0.154} & 0.151 & 0.145 & 0.132 \\
\bottomrule
\end{tabular}%
}
\caption{Impact of amplification factor $\alpha$ on task performance. All configurations assign advantage 1 to trajectories achieving maximum reward in the group, ensuring positive signals even at $\alpha=1.0$.}
\label{tab:alpha_analysis}
\end{table}

Table \ref{tab:alpha_analysis} demonstrates the impact of $\alpha$. Even at $\alpha=1.0$, R$^3$L achieves strong performance because trajectories with maximum reward receive a fixed advantage of 1, guaranteeing constructive learning signals regardless of amplification strength. This baseline amplification explains why $\alpha=1.0$ already performs well on tasks like GSM8K, where it achieves 0.732. 

Increasing $\alpha$ further amplifies the contribution of positive-advantage trajectories beyond just the best one. Agentic environments benefit from this additional amplification, with WebShop improving from 0.647 at $\alpha=1.0$ to 0.663 at $\alpha=3.0$ and ScienceWorld from 0.375 to 0.385. Mathematical reasoning tasks show more stability across moderate $\alpha$ values, with GSM8K performing best at $\alpha=1.0$ while Math500 peaks at $\alpha=3.0$. Pushing $\alpha$ beyond 5.0 causes degradation on most tasks due to overfitting on specific retry trajectories. The choice $\alpha=3.0$ provides robust performance across diverse tasks, consistent with the theoretical analysis in Appendix \ref{sec:gradient_dominance}.

\textbf{Complexity drop: can Positive Amplification alone match full R$^3$L?} A natural follow-up is whether Positive Amplification, added to standard GRPO without Reflect-then-Retry or Pivotal Credit Assignment, can close the gap to full R$^3$L. Table~\ref{tab:retry_count} provides this decomposition directly: GRPO+PA (row ``0 retries'', all 8 base trajectories) reaches 0.807 on ALFWorld, a meaningful lift over GRPO's 0.720 in Table~\ref{tab:main_results} but still well short of R$^3$L's 0.928. Positive Amplification therefore contributes roughly 40\% of R$^3$L's total gain over GRPO on this benchmark, while Reflect-then-Retry and Pivotal Credit Assignment together account for the remaining 60\%. Positive Amplification is a necessary stabilizer for learning from retry data, but not by itself a sufficient driver of R$^3$L's gains.

\subsection{Synchronization Frequency}
A critical challenge in iterative reinforcement learning is the distributional drift between the behavior policy $\pi_{\theta_{old}}$ and the learner policy $\pi_{\theta}$. Frequent synchronization is computationally expensive, making robustness to stale reference policies desirable. We evaluate R$^3$L against GRPO, OPMD, and GRPO+Positive, which adds Positive Amplification to standard GRPO.

\begin{table}[h]
\centering
\adjustbox{max width=\linewidth}{%
\begin{tabular}{llcccc}
\toprule
\textbf{Task} & \textbf{Method} & $S=1$ & $S=5$ & $S=10$ & $S=20$ \\
\midrule
\multirow{4}{*}{ALFWorld} 
& OPMD & 0.835 & 0.357$^\dagger$ & 0.257$^\dagger$ & 0.341$^\dagger$ \\
& GRPO & 0.720 & 0.716 & 0.389$^\dagger$ & 0.742 \\
& GRPO+Positive & 0.807 & 0.864 & 0.778 & 0.764 \\
& \textbf{R$^3$L} & \textbf{0.928} & \textbf{0.945} & \textbf{0.922} & \textbf{0.934} \\
\midrule
\multirow{4}{*}{WebShop} 
& OPMD & 0.561 & 0.619 & 0.518$^\dagger$ & 0.460$^\dagger$ \\
& GRPO & 0.614 & 0.638 & 0.612 & 0.610 \\
& GRPO+Positive & 0.648 & 0.633 & 0.632 & 0.622 \\
& \textbf{R$^3$L} & \textbf{0.663} & \textbf{0.663} & \textbf{0.672} & \textbf{0.657} \\
\midrule
\multirow{4}{*}{GSM8K} 
& OPMD & 0.463$^\dagger$ & 0.604$^\dagger$ & 0.655$^\dagger$ & 0.499$^\dagger$ \\
& GRPO & 0.474 & 0.673 & 0.692 & 0.712 \\
& GRPO+Positive & 0.504 & 0.520 & 0.714 & 0.713 \\
& \textbf{R$^3$L} & \textbf{0.721} & \textbf{0.779} & \textbf{0.756} & \textbf{0.753} \\
\midrule
\multirow{4}{*}{Math500} 
& OPMD & 0.292$^\dagger$ & 0.354$^\dagger$ & 0.332$^\dagger$ & 0.336$^\dagger$ \\
& GRPO & 0.368 & 0.406 & 0.408 & 0.406 \\
& GRPO+Positive & 0.377 & 0.409 & 0.411 & 0.406 \\
& \textbf{R$^3$L} & \textbf{0.424} & \textbf{0.464} & \textbf{0.441} & \textbf{0.426} \\
\bottomrule
\end{tabular}%
}
\caption{Effect of synchronization frequency $S$ across benchmarks. $S$ denotes the number of training steps between behavior policy updates. $^\dagger$ denotes collapsed training where we report the peak score before collapse.}
\label{tab:sync_analysis}
\end{table}

Table \ref{tab:sync_analysis} reveals the stability characteristics of different methods under varying synchronization frequencies. Increasing $S$ enlarges the gap between the behavior policy and the learner policy, introducing more off-policy data into training and increasing the risk of instability. OPMD is highly sensitive to this risk, collapsing on ALFWorld from 0.835 at $S=1$ to 0.257 at $S=10$, and fluctuating erratically on GSM8K between 0.463 and 0.655 before dropping to 0.499 at $S=20$. Standard GRPO shows task-dependent sensitivity, improving on GSM8K as $S$ increases but collapsing on ALFWorld at $S=10$. GRPO+Positive demonstrates that Positive Amplification mitigates the instability risk, achieving 0.504 at $S=1$ on GSM8K and improving to 0.714 at $S=10$. R$^3$L achieves the highest absolute performance across all tasks and synchronization intervals while maintaining stability. On ALFWorld, R$^3$L stays above 0.920 across all $S$ values. On GSM8K, it reaches 0.779 at $S=5$ and maintains 0.753 at $S=20$, consistently outperforming all baselines.

\subsection{Retry Count}
\label{sec:retry_count}

A natural question is whether more retries per base trajectory yield better exploration. We fix the total trajectory budget at $N=8$ and vary the retry count per base: 0 retries (all 8 base), 1 retry per base (4 base + 4 retry, R$^3$L's default), 2 sequential retries per base ($4 \times 2$), and 4 sequential retries per base ($4 \times 1$). Each additional retry restarts from a pivot identified by a fresh reflection on the most recent attempt.

\begin{table}[t]
\centering
\small
\adjustbox{max width=\linewidth}{%
\begin{tabular}{l l c}
\toprule
\textbf{Retry count} & \textbf{Trajectory allocation} & \textbf{ALFWorld} \\
\midrule
0 (GRPO + PA) & 8 base                  & 0.807 \\
\rowcolor{cyan!10} 1 (R$^3$L default) & 4 base + 4 retry        & \textbf{0.928} \\
2             & 4 base + 2$\times$2 retry & 0.913 \\
4             & 4 base + 4$\times$1 retry & 0.896 \\
\bottomrule
\end{tabular}%
}
\caption{Retry count ablation on Qwen2.5-1.5B, ALFWorld, with total trajectory budget $N=8$. All multi-retry settings keep 4 base trajectories and 4 retry trajectories. The notation $a \times b$ in the allocation column denotes $a$ sequential retries applied to each of $b$ base trajectories (so $a \times b = 4$ retry trajectories total); each additional retry restarts from the pivot identified by a fresh reflection on the previous attempt.}
\label{tab:retry_count}
\end{table}

Table~\ref{tab:retry_count} shows that a single retry already captures the bulk of R$^3$L's advantage over GRPO+PA, lifting ALFWorld from 0.807 to 0.928, while additional retries actually hurt performance: 2 retries drop to 0.913 and 4 retries to 0.896. Two factors explain this. First, under the fixed $N=8$ budget, more retries per base means fewer distinct base trajectories, reducing the diversity of base-retry pairs that feed GRPO's contrastive signal. Second, each reflection is an independent one-shot diagnosis that does not accumulate experience across rounds; the useful information comes from environment feedback after the first retry rather than from iterating on the same trajectory. This justifies R$^3$L's choice of exactly one retry per base.

\subsection{Group Size}
\label{sec:group_size}

Group size $N$ controls the number of trajectories per query used for group-relative advantage normalization. We compare R$^3$L, GRPO, Reflect-GRPO, and Critique-GRPO under $N \in \{4, 8, 16\}$ on Qwen2.5-1.5B across five benchmarks spanning agentic and mathematical reasoning.

\begin{table}[t]
\centering
\small
\adjustbox{max width=\linewidth}{%
\begin{tabular}{ll cccc}
\toprule
\textbf{$N$} & \textbf{Task} & \textbf{GRPO} & \textbf{Reflect-GRPO} & \textbf{Critique-GRPO} & \textbf{R$^3$L} \\
\midrule
\multirow{5}{*}{4}
 & ALFWorld     & 0.742 & 0.774 & 0.816 & \textbf{0.886} \\
 & ScienceWorld & 0.332 & 0.352 & 0.368 & \textbf{0.438} \\
 & WebShop      & 0.588 & 0.600 & 0.614 & \textbf{0.624} \\
 & Math500      & 0.334 & 0.348 & 0.362 & \textbf{0.398} \\
 & DAPO         & 0.108 & 0.118 & 0.116 & \textbf{0.138} \\
\midrule
\multirow{5}{*}{8}
 & ALFWorld     & 0.780 & 0.812 & 0.854 & \textbf{0.928} \\
 & ScienceWorld & 0.366 & 0.388 & 0.406 & \textbf{0.482} \\
 & WebShop      & 0.620 & 0.634 & 0.648 & \textbf{0.663} \\
 & Math500      & 0.367 & 0.382 & 0.398 & \textbf{0.432} \\
 & DAPO         & 0.123 & 0.136 & 0.133 & \textbf{0.156} \\
\midrule
\multirow{5}{*}{16}
 & ALFWorld     & 0.788 & 0.818 & 0.860 & \textbf{0.932} \\
 & ScienceWorld & 0.372 & 0.394 & 0.412 & \textbf{0.486} \\
 & WebShop      & 0.628 & 0.640 & 0.654 & \textbf{0.668} \\
 & Math500      & 0.374 & 0.388 & 0.404 & \textbf{0.436} \\
 & DAPO         & 0.128 & 0.140 & 0.137 & \textbf{0.160} \\
\bottomrule
\end{tabular}%
}
\caption{Group size $N$ ablation on Qwen2.5-1.5B across five benchmarks. \textbf{Bold} marks the best method per row. R$^3$L leads on every $(N, \text{benchmark})$ setting; gains are largest moving from $N=4$ to $N=8$, while $N=8 \to 16$ yields only marginal improvement.}
\label{tab:group_size}
\end{table}

Table~\ref{tab:group_size} shows that R$^3$L leads on every $(N, \text{benchmark})$ setting. Gains are largest moving from $N=4$ to $N=8$ (e.g., ALFWorld $0.886 \to 0.928$, ScienceWorld $0.438 \to 0.482$), while $N=8 \to 16$ yields only marginal improvement (ALFWorld $0.928 \to 0.932$). The same diminishing-returns pattern holds across all four methods, indicating $N=8$ is a practical sweet spot: it provides sufficient group diversity for stable variance-based normalization without excessive rollout cost. R$^3$L's relative advantage is preserved across $N$, confirming that its gains do not depend on a specific group size.

\section{Compute Efficiency Analysis}
\label{sec:compute}

Appendix~\ref{app:implementation} reports overall GPU hours; here we provide per-training-step cost breakdowns across four representative model--task settings, addressing how R$^3$L's design choices translate into wall-clock and token budgets. All measurements use two 8-GPU H20 servers with asynchronous pipeline parallelism.

\subsection{Structural Cost Comparison}

\begin{table*}[t]
\centering
\small
\resizebox{\textwidth}{!}{%
\begin{tabular}{l llll}
\toprule
& \textbf{GRPO} & \textbf{Reflect-GRPO} & \textbf{Critique-GRPO} & \textbf{R$^3$L} \\
\midrule
Base rollout & $N$ full & $N$ full & $N$ full & $N/2$ full \\
Reflect & --- & in-trajectory tokens & $N$ independent critiques & $N/2$ JSON ($\sim$500 tok) \\
Retry rollout & --- & in-trajectory & $N$ full refinements & $N/2$ partial from pivot$^\dagger$ \\
Total rollouts & $N$ & $N$ & $2N$ & $N$ to $2N^\dagger$ \\
Enters training & $N$ trajectories & $N$ with reflect tokens & $N$ initial + best refinement & base + distilled, all in group \\
Ref.\ model forward & required (KL) & required (KL) & required (KL) & not required \\
\bottomrule
\end{tabular}%
}
\caption{Rollout structure per training step. Critique-GRPO issues $2N$ rollouts and discards $N{-}1$ refinements; R$^3$L splits $N$ rollouts into $N/2$ base and $N/2$ conditionally triggered partial retries, and removes the reference-model forward pass by dropping KL. $^\dagger$R$^3$L's retry is triggered only when reflection deems the base trajectory non-successful, so total rollouts lie between $N$ and $2N$ depending on the trigger rate.}
\label{tab:compute_structural}
\end{table*}

Table~\ref{tab:compute_structural} summarizes the rollout structure of each method. Two design differences drive the cost gap. First, Critique-GRPO issues $2N$ rollouts ($N$ base plus $N$ full refinements) and selects only the best refinement for training, discarding computation for the remaining $N{-}1$; R$^3$L splits the same $N$ rollouts into $N/2$ base and $N/2$ partial retries triggered only when reflection deems the base non-successful, with all trajectories contributing to the training group. Second, R$^3$L drops the KL penalty and therefore removes the reference-model forward pass that the other three methods require.

\subsection{Per-step Measurements}

\begin{table*}[t]
\centering
\small
\resizebox{\textwidth}{!}{%
\begin{tabular}{l l cccc}
\toprule
\textbf{Setting} & \textbf{Metric} & \textbf{GRPO} & \textbf{Reflect-GRPO} & \textbf{Critique-GRPO} & \textbf{R$^3$L} \\
\midrule
\multirow{4}{*}{Qwen2.5-1.5B on ALFWorld}
 & Total rollout turns / sample  & 114.4   & 146.8   & 208.0   & \textbf{75.2}  \\
 & Total training tokens / sample & 44{,}141 & 52{,}400 & 36{,}200 & \textbf{22{,}300} \\
 & Rollout time / step (s)       & 406.2   & 492.8   & 756.3   & \textbf{372.4} \\
 & Train time / step (s)         & \textbf{87.6} & 109.4 & 94.2  & 103.8 \\
\midrule
\multirow{4}{*}{Qwen2.5-1.5B on DAPO}
 & Total rollout turns / sample  & \textbf{10.2}    & 19.6    & 28.1    & 12.9    \\
 & Total training tokens / sample & \textbf{16{,}400} & 27{,}600 & 17{,}300 & 18{,}000 \\
 & Rollout time / step (s)       & 342.6   & 418.3   & 646.8   & \textbf{292.4} \\
 & Train time / step (s)         & \textbf{54.8} & 73.6    & 59.2    & 87.4    \\
\midrule
\multirow{4}{*}{Qwen3-4B on ALFWorld}
 & Total rollout turns / sample  & 81.6    & 108.0   & 153.6   & \textbf{58.8}  \\
 & Total training tokens / sample & 70{,}176 & 83{,}900 & 50{,}600 & \textbf{26{,}900} \\
 & Rollout time / step (s)       & 713.8   & 836.4   & 1356.8  & \textbf{670.4} \\
 & Train time / step (s)         & \textbf{188.9} & 218.7 & 197.3 & 208.6 \\
\midrule
\multirow{4}{*}{Qwen3-4B on DAPO}
 & Total rollout turns / sample  & \textbf{8.72}    & 16.6    & 25.12   & 11.24   \\
 & Total training tokens / sample & 28{,}663 & 42{,}200 & 30{,}874 & \textbf{25{,}500} \\
 & Rollout time / step (s)       & 1148.0  & 1392.6  & 2318.4  & \textbf{773.5} \\
 & Train time / step (s)         & \textbf{121.3} & 146.8 & 129.6 & 225.8 \\
\bottomrule
\end{tabular}%
}
\caption{Per-training-step cost breakdown on four representative model--task settings, measured on two 8-GPU H20 servers with asynchronous pipeline parallelism. Token counts are model-generated response tokens only (prompt tokens excluded). \textbf{Bold} marks the lowest value in each row.}
\label{tab:compute_measurements}
\end{table*}

Table~\ref{tab:compute_measurements} reports per-step rollout and training cost on four representative settings. Three observations follow.

\textbf{Rollout time: R$^3$L is the lowest across all four settings.} On the most expensive setting (Qwen3-4B on DAPO), R$^3$L completes at 773.5\,s/step versus GRPO's 1148.0\,s ($-33\%$) and Critique-GRPO's 2318.4\,s ($-67\%$). The saving comes from a combination of fewer rollout turns on multi-step tasks and no reference-model forward pass.

\textbf{Training tokens: roughly halved on multi-step tasks.} R$^3$L uses 22{,}300 tokens/sample on Qwen2.5-1.5B ALFWorld versus GRPO's 44{,}141 ($-49\%$), and 26{,}900 versus 70{,}176 on Qwen3-4B ALFWorld ($-62\%$). The reduction stems from pivot-based partial retry combined with the conditional retry trigger: once the base trajectory succeeds, no retry is launched and no extra tokens enter training.

\textbf{Single-step math tasks: rollout turns comparable, but rollout wall-clock still lower.} On DAPO the trajectory is already one turn, so R$^3$L's turn savings vanish (Qwen2.5-1.5B: GRPO 10.2 vs R$^3$L 12.9; Qwen3-4B: 8.72 vs 11.24). R$^3$L still wins on rollout time (292.4 vs 342.6 on Qwen2.5-1.5B; 773.5 vs 1148.0 on Qwen3-4B) because the KL-free objective removes the reference-model forward pass.

R$^3$L's per-step train time can be slightly higher than GRPO's on some settings (e.g., Qwen3-4B on DAPO: 225.8 vs 121.3) because it runs auxiliary SFT on reflection and retry data. The overall wall-clock benefit nevertheless favours R$^3$L, since rollout time dominates total per-step cost and R$^3$L leads on rollout time across all four settings.

\section{Comparison with Related Methods}
\label{sec:related_comparison}

Section~\ref{sec:related_work} positions R$^3$L against existing baselines; here we consolidate the key design distinctions along four axes.

For exploration, GRPO~\cite{shao2024deepseekmath}, GSPO~\cite{zheng2025group}, BAPO~\cite{xi2025bapo}, and OPMD~\cite{yao2025group} rely on pure stochastic sampling. Reflect-GRPO~\cite{bensal2025reflect} and Critique-GRPO~\cite{zhang2025critique} supplement sampling with retries, but trigger retry on every base trajectory. R$^3$L instead triggers retry only when reflection classifies the base as non-successful, and restarts from the identified pivot rather than from scratch.

For credit assignment, GRPO-family methods assign trajectory-level advantages. Process Reward Models~\cite{wang2024math,lightman2023let} offer step-level supervision but require expensive human annotations; GiGPO~\cite{feng2025group} and VinePPO~\cite{kazemnejad2024vineppo} estimate step-level credit through anchor states or Monte Carlo rollouts, adding compute or state-compression constraints. R$^3$L derives step-level credit from the contrastive structure of base-retry pairs without any external annotation, learned verifier, or additional rollouts.

For rollout cost, Critique-GRPO doubles the budget to $2N$ and discards $N{-}1$ refinements. R$^3$L keeps the budget within $N$ ($N/2$ base plus conditional $N/2$ partial retries) and further removes the reference-model forward pass by dropping the KL penalty (Appendix~\ref{sec:compute}).

For hyperparameters, GSPO requires an adaptive clipping range, BAPO uses multiple dynamic clipping bounds, and Critique-GRPO adds a shaping coefficient. R$^3$L adds a single scalar $\alpha$ fixed at $3.0$ across all experiments, and drops both $\beta$ (KL) and $\epsilon$ (importance-sampling clipping) from GRPO.

R$^3$L is the only method in our comparison to simultaneously achieve step-level credit, require no external supervision, stay within the GRPO rollout budget, and introduce a single new hyperparameter.

\section{Algorithm}
\label{sec:algorithm}
\begin{algorithm*}[t]
\small
\caption{R$^3$L: Reflect-then-Retry Reinforcement Learning with Language-Guided Exploration, Pivotal Credit and Positive Amplification}
\label{alg:R3L}
\begin{algorithmic}[1]
\STATE \textbf{Input:} Policy $\pi_{\theta}$, Reward function $R$, Prompt dataset $\mathcal{D}$
\STATE \textbf{Hyperparameters:} Group size $N$, Amplification factor $\alpha$
\STATE \textbf{Initialize:} $\theta \leftarrow \theta_{old}$
\WHILE{not converged}
    \FOR{each prompt $x$ in batch $\mathcal{B} \sim \mathcal{D}$}
        \STATE $\mathcal{G} \leftarrow \emptyset$, $\mathcal{D}_{reflect} \leftarrow \emptyset$, $\mathcal{D}_{retry} \leftarrow \emptyset$ \hfill $\triangleright$ Initialize exploration group and SFT datasets
        
        \STATE \textcolor{blue}{\textit{// Phase 1: Language-Guided Exploration}}
        \STATE Sample $N/2$ base trajectories: $\{\tau_1, \dots, \tau_{N/2}\} \sim \pi_{\theta_{old}}(\cdot|x)$
        \FOR{each base trajectory $\tau_i$ with environment feedback $f_i$}
            \STATE Generate reflection $r_i \sim \pi_{\theta_{old}}(\cdot | \tau_i, f_i, \text{prompt}_{reflect})$
            \STATE Parse $r_i$ to obtain pivot turn $k_{pivot}$ and guidance $g_i$
            \STATE Sample retry suffix from $k_{pivot}$: $\tau'_{i,\geq k_{pivot}} \sim \pi_{\theta_{old}}(\cdot | \tau_{i,<k_{pivot}}, g_i)$
            \STATE \textbf{Context Distillation:} Construct $\tau'_i = (\tau_{i,<k_{pivot}}, \tau'_{i,\geq k_{pivot}})$ without $g_i$
            
            \STATE \textcolor{blue}{\textit{// Phase 2: Pivotal Credit Assignment}}
            \STATE Define pivot mask: $M[k] = 0$ if $k < k_{pivot}$ else $1$
            \STATE Add $(\tau_i, M)$ to $\mathcal{G}$
            \STATE Add $(\tau'_i, M)$ to $\mathcal{G}$
            
            \STATE \textcolor{blue}{\textit{// Collect SFT Data for Meta-Tasks (verified successful corrections)}}
            \IF{$R(\tau'_i) > R(\tau_i)$}
                \STATE Add $([\tau_i, f_i], r_i)$ to $\mathcal{D}_{reflect}$ \hfill $\triangleright$ Learn to reflect
                \STATE Add $(\tau_{i,<k_{pivot}} \oplus g_i, \tau'_{i,\geq k_{pivot}})$ to $\mathcal{D}_{retry}$ \hfill $\triangleright$ Learn to retry
            \ENDIF
        \ENDFOR
        
        \STATE \textcolor{blue}{\textit{// Phase 3: Positive Amplification}}
        \STATE Compute rewards $\{R(\tau) | \tau \in \mathcal{G}\}$
        \STATE Compute group statistics: $\bar{R} = \text{mean}$, $\sigma_R = \text{std}$, $R_{max} = \max$
        \FOR{each $(\tau_j, M_j) \in \mathcal{G}$}
            \STATE $A_j \leftarrow (R(\tau_j) - \bar{R}) / \sigma_R$
            \IF{$R(\tau_j) = R_{max}$}
                \STATE $\hat{A}_j \leftarrow \alpha$
            \ELSIF{$A_j > 0$}
                \STATE $\hat{A}_j \leftarrow \alpha \cdot A_j$
            \ELSE
                \STATE $\hat{A}_j \leftarrow A_j$
            \ENDIF
        \ENDFOR
        
        \STATE \textcolor{blue}{\textit{// Policy Update with RL and Auxiliary SFT Objectives}}
        \STATE $\mathcal{L}_{RL}(\theta) = -\frac{1}{|\mathcal{G}|} \sum_{(\tau_j, M_j) \in \mathcal{G}} \frac{1}{|\tau_j|} \sum_{k,t} M_j[k] \cdot \hat{A}_j \cdot \log \pi_{\theta}(y_k^t | h_k, y_k^{<t})$
        \STATE $\mathcal{L}_{SFT}(\theta) = - \sum_{(x,y) \in \mathcal{D}_{reflect} \cup \mathcal{D}_{retry}} \log \pi_{\theta}(y|x)$ \hfill $\triangleright$ Maintain reflection and retry skills
        \STATE Update $\theta$ via gradient descent on $\mathcal{L}_{RL}(\theta) + \mathcal{L}_{SFT}(\theta)$
    \ENDFOR
\ENDWHILE
\end{algorithmic}
\end{algorithm*}
Algorithm \ref{alg:R3L} summarizes the complete R$^3$L training procedure. The algorithm addresses a fundamental tension in reinforcement fine-tuning: discovering successful trajectories requires exploration, but learning from them requires stable optimization. R$^3$L resolves this tension through three integrated phases.

\textbf{Phase 1: Language-Guided Exploration.} Standard methods allocate their full sampling budget to stochastic exploration, which yields predominantly failed trajectories on difficult tasks. R$^3$L instead splits the budget evenly between base sampling and guided retry. For each base trajectory, the model generates a structured reflection that diagnoses errors and identifies the pivot turn where the issue first manifested. Generation then restarts from this pivot, conditioned on the diagnostic guidance, producing a retry suffix that addresses the identified failure. The critical step is context distillation, which combines the original prefix with the corrected suffix while removing the guidance from the training input. This forces the model to internalize corrections rather than relying on explicit diagnostic prompts, ensuring that improvements transfer directly to inference where no guidance is available.

\textbf{Phase 2: Pivotal Credit Assignment.} The exploration phase produces base and retry trajectories that share identical prefixes up to the pivot turn. This contrastive structure enables precise credit assignment without external annotation. We construct a binary mask that zeros out all turns before the pivot, excluding the shared prefix from gradient computation. Both base and retry trajectories receive the same mask, focusing optimization exclusively on the diverging suffix where one path succeeds, and the other fails. This design serves two purposes. First, it protects valid reasoning in the prefix from being penalized when a later error causes trajectory failure. Second, it reduces gradient variance by eliminating the noise contribution from tokens where both trajectories behave identically.

\textbf{Phase 3: Positive Amplification.} The exploration group now contains both base and retry trajectories with varying rewards. On difficult tasks, failures often outnumber successes even after retries, causing standard group-relative normalization to dilute positive signals. We counteract this by reweighting advantages before computing the policy gradient. Trajectories achieving the maximum reward in the group receive the full amplification factor $\alpha$, ensuring at least one strong positive signal per group. Other positive-advantage trajectories are scaled proportionally to maintain their relative ordering, while negative-advantage trajectories remain unchanged. This asymmetric treatment guarantees that constructive gradients dominate the update direction, channeling probability mass toward discovered solutions rather than dispersing it through blind error suppression.

The final policy update combines the pivot mask with amplified advantages in a single gradient step. Unlike standard GRPO, R$^3$L omits both importance sampling and KL regularization. Importance sampling becomes unreliable for retry trajectories generated under guidance, as the behavior distribution differs fundamentally from the current policy. KL constraints are unnecessary because positive amplification already prevents entropy collapse. This simplification reduces computational and memory overhead while the three-phase design maintains training stability.

\section{Trajectory Visualization and Context Distillation}
\label{sec:appendix_trajectory}

This section details the four trajectory types generated during R$^3$L training and analyzes how context distillation enables learning from guided exploration without creating inference-time dependencies. Figure \ref{fig:trajectory_types} visualizes these four categories.

\begin{figure*}[htbp]
    \centering
    \includegraphics[width=1.0\linewidth]{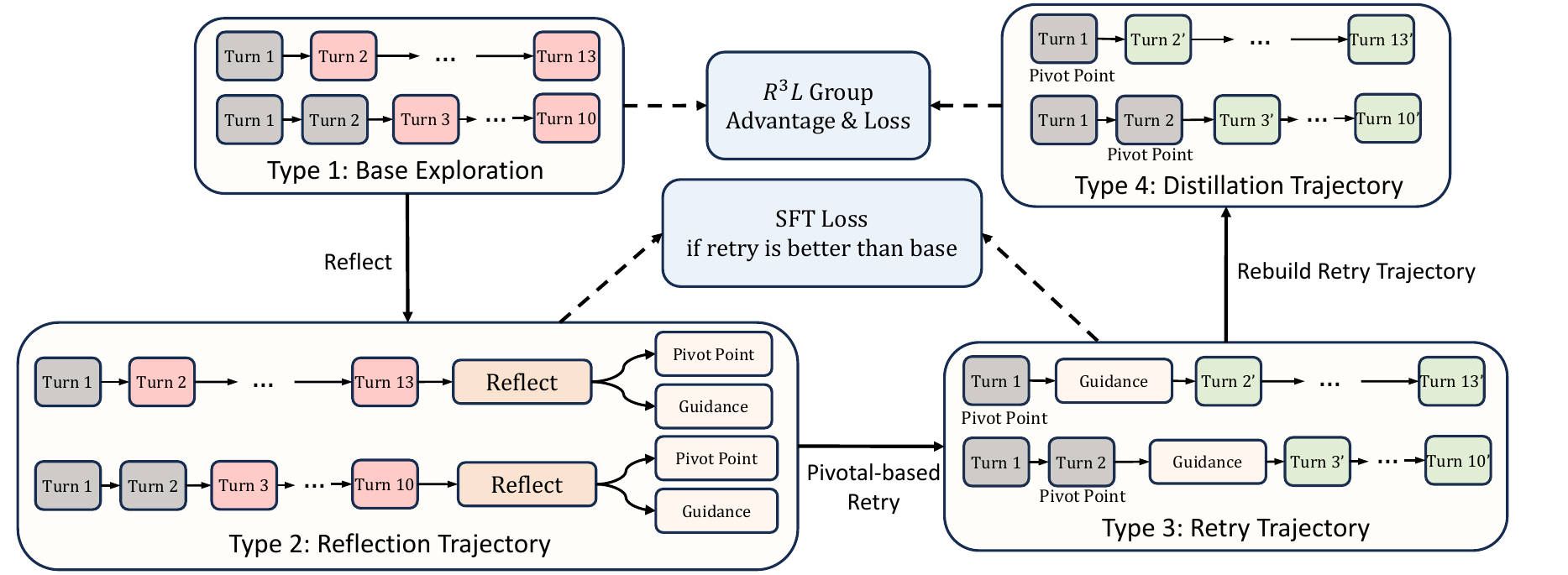}
    \caption{Four types of trajectories in R$^3$L. Type 1 represents base exploration from the current policy. Type 2 captures the reflection process that diagnoses errors and identifies pivot points. Type 3 shows retry generation conditioned on diagnostic guidance. Type 4 is the distillation trajectory that combines the original prefix with the corrected suffix, removing guidance dependency for training.}
    \label{fig:trajectory_types}
\end{figure*}

\subsection{Four Trajectory Types}

\textbf{Type 1: Base Trajectories.} Standard exploration samples trajectories directly from the behavior policy $\pi_{\theta_{old}}(\cdot | x)$. These trajectories form $\mathcal{D}_{base}$ and enter the exploration group $\mathcal{G}_{explore}$ for RL optimization.

\textbf{Type 2: Reflection Trajectories.} Given a base trajectory paired with environment feedback, the model generates a structured diagnosis including outcome classification, root cause analysis, improvement suggestions, and the pivot turn $k_{pivot}$ where the issue first manifested. This reflection transforms implicit failure signals into explicit actionable guidance.

\textbf{Type 3: Retry Trajectories.} Generation restarts from $k_{pivot}$ conditioned on the diagnostic guidance $g$, producing a corrected suffix $\tau'_{\geq k_{pivot}}$. Here, we refer to the turns preceding the pivot as the \emph{prefix} ($\tau_{<k_{pivot}}$) and the turns from the pivot onward as the \emph{suffix} ($\tau_{\geq k_{pivot}}$); both credit assignment and trajectory stitching operate at this turn level. Conditioning on explicit error analysis significantly increases the likelihood of generating improved solutions compared to unguided sampling.

\textbf{Type 4: Distillation Trajectories.} We construct distillation trajectories by pairing the original prefix $\tau_{<k_{pivot}}$ with the retry suffix $\tau'_{\geq k_{pivot}}$, explicitly removing the guidance $g$ from the context. These trajectories form $\mathcal{D}_{distill}$ and enter the exploration group alongside base trajectories for RL optimization. This construction is essential because directly training on Type 3 would teach the model to expect guidance at inference time when none is available.

\subsection{Training Data Organization}

The exploration group $\mathcal{G}_{explore} = \mathcal{D}_{base} \cup \mathcal{D}_{distill}$ contains both base and distillation trajectories. All trajectories receive rewards from the environment, and advantages are computed through group-relative normalization followed by Positive Amplification. This means distillation trajectories that improve upon their base counterparts receive positive advantages, while those that fail to improve receive negative advantages and are suppressed.

Separately, we construct two auxiliary SFT datasets from instances where retry achieves strictly higher reward than the base attempt. The reflection dataset $\mathcal{D}_{reflect}$ trains the model to produce structured diagnoses, and the retry dataset $\mathcal{D}_{retry}$ trains guided correction with the guidance present in the input. These auxiliary tasks maintain reflection and retry capabilities throughout training, preventing skill degradation as the policy distribution shifts.

\subsection{Context Distillation Mechanism}
The key insight is that Type 4 trajectories enable the model to learn from guided exploration without depending on guidance at inference time. Standard context distillation trains a student to match a teacher's output distribution when the teacher has access to additional context:
\begin{equation}
\mathcal{L}_{distill} = -\mathbb{E}_{y \sim \pi_{teacher}(\cdot | x, g)} \left[ \log \pi_{\theta}(y | x) \right]
\end{equation}

R$^3$L achieves a similar effect through its RL objective. For a distillation trajectory $\tau' = (\tau_{<k_{pivot}}, \tau'_{\geq k_{pivot}})$, the gradient takes the form:
\begin{equation}
\nabla \mathcal{L} \propto \hat{A}(\tau') \cdot \nabla \log \pi_{\theta}(\tau'_{\geq k_{pivot}} | \tau_{<k_{pivot}})
\end{equation}
where $\log \pi_{\theta}(\tau'_{\geq k_{pivot}} | \tau_{<k_{pivot}})$ denotes the autoregressive log-probability of the retry suffix conditioned on the original prefix. When $\tau'$ achieves high reward, positive $\hat{A}(\tau')$ increases the likelihood of generating the corrected suffix without guidance. When retries fail to improve, negative advantages suppress those actions. The result is amortized inference: the model learns to generate improved solutions directly, having internalized correction patterns through training.

\section{Analysis of the Pivotal Mechanism}
\label{sec:pivot_analysis}

\subsection{Pivot Location Evolution}

The pivotal mechanism identifies failure points and focuses learning on critical decision turns. Here, we examine how pivot point locations change as training progresses.

\begin{figure}[t]
    \centering
    \resizebox{0.75\linewidth}{!}{%
    \begin{minipage}{\linewidth}
    \begin{minipage}{0.49\linewidth}
        \centering
        \begin{tikzpicture}
            \begin{axis}[
                width=\linewidth, height=4.7cm,
                xlabel={Training Steps}, ylabel={Avg. Pivot Point},
                grid=major, ymin=0, ymax=8, ytick={0,2,4,6,8},
                font=\footnotesize, yticklabel={\pgfmathprintnumber{\tick}},
            ]
            \addplot[color=blue, mark=square*, mark size=0pt, line width=1.2pt, smooth, tension=1.0]
            table[x=step, y=pivot] {figure/src/alfworld-pivot.dat};
            \end{axis}
        \end{tikzpicture}
        \centerline{\footnotesize (a) ALFWorld}
    \end{minipage}\hfill
    \begin{minipage}{0.49\linewidth}
        \centering
        \begin{tikzpicture}
            \begin{axis}[
                width=\linewidth, height=4.7cm,
                xlabel={Training Steps}, ylabel={Avg. Pivot Point},
                grid=major, ymin=0,
                font=\footnotesize, yticklabel={\pgfmathprintnumber{\tick}},
            ]
            \addplot[color=green!70!black, mark=triangle*, mark size=0pt, line width=1.2pt, smooth, tension=1.0]
            table[x=step, y=pivot] {figure/src/webshop-pivot.dat};
            \end{axis}
        \end{tikzpicture}
        \centerline{\footnotesize (b) WebShop}
    \end{minipage}

    \vspace{0.6em}

    \begin{minipage}{0.49\linewidth}
        \centering
        \begin{tikzpicture}
            \begin{axis}[
                width=\linewidth, height=4.7cm,
                xlabel={Training Steps}, ylabel={Avg. Pivot Point},
                grid=major, ymin=0, ymax=15, ytick={0,3,6,9,12,15},
                font=\footnotesize, yticklabel={\pgfmathprintnumber{\tick}},
            ]
            \addplot[color=red, mark=*, mark size=0pt, line width=1.2pt, smooth, tension=1.0]
            table[x=step, y=pivot] {figure/src/scienceworld-pivot.dat};
            \end{axis}
        \end{tikzpicture}
        \centerline{\footnotesize (c) ScienceWorld}
    \end{minipage}\hfill
    \begin{minipage}{0.49\linewidth}
        \centering
        \begin{tikzpicture}
            \begin{axis}[
                width=\linewidth, height=4.7cm,
                xlabel={Training Steps}, ylabel={Avg. Pivot Point},
                grid=major, ymin=0,
                font=\footnotesize, yticklabel={\pgfmathprintnumber{\tick}},
            ]
            \addplot[color=orange, mark=diamond*, mark size=0pt, line width=1.2pt, smooth, tension=1.0]
            table[x=step, y=pivot] {figure/src/dapo-pivot.dat};
            \end{axis}
        \end{tikzpicture}
        \centerline{\footnotesize (d) DAPO (Math)}
    \end{minipage}
    \end{minipage}}
    \caption{Evolution of average pivot points across training. The pivot point indicates the step where the model's initial trajectory fails and requires correction. An increasing trend suggests the model learns to succeed at earlier stages, pushing failure points further into the trajectory.}
    \label{fig:pivot_evolution}
\end{figure}

Figure \ref{fig:pivot_evolution} plots the average pivot point $k_{pivot}$ across training steps for four benchmarks, where $k_{pivot} = 0$ indicates restarting from the beginning of the trajectory. We observe consistent upward trends in embodied environments. ALFWorld pivot points increase from approximately 2 to over 6 steps, while ScienceWorld exhibits an even more pronounced shift from around 2 to over 12 steps across 500 training steps. WebShop shows rapid initial growth from 1 to 4 steps within the first 20 steps, but then declines and stabilizes around 1. For mathematical reasoning in DAPO, which allows up to 3 attempts, the pivot point rises from 0.4 to 0.6 with considerable variance.

This rightward shift carries important implications. Early in training, pivot points concentrate near trajectory beginnings, indicating the model struggles with initial planning or format adherence. As training progresses, average pivot positions migrate toward later stages, suggesting the model stabilizes early-stage actions and preserves valid prefixes. The pivotal mechanism's role consequently evolves from global trajectory correction to fine-grained refinement at execution tails. This evolution also strengthens variance reduction by Theorem \ref{thm:variance_reduction}. As $T_{pivot}/T$ increases, gradient variance decreases proportionally.

\subsection{Pivot Localization Accuracy}
\label{sec:pivot_accuracy}

Beyond the evolution of pivot locations, we directly evaluate how accurately the model identifies failure points. ScienceWorld provides step-level subtask rewards, enabling automatic oracle pivot identification. For ALFWorld and WebShop, we sample 100 failed trajectories per benchmark and annotate oracle pivots through a combination of human judgment and DeepSeek labeling. We define Oracle Agreement as the model's pivot falling within one step of the oracle position, and classify each diagnosis as Accurate, Partial (within 3 steps), or Misdiagnosed.

\begin{table}[t]
\centering
\small
\adjustbox{max width=\linewidth}{%
\begin{tabular}{ll rrr rrr}
\toprule
\multirow{2}{*}{\textbf{Task (Model)}} & \multirow{2}{*}{\textbf{Step}}
 & \textbf{Oracle} & \textbf{Retry SR} & \textbf{Retry SR}
 & \multirow{2}{*}{\textbf{Accurate}} & \multirow{2}{*}{\textbf{Partial}} & \multirow{2}{*}{\textbf{Misdiag.}} \\
 & & \textbf{Agr.} & \textbf{(Correct)} & \textbf{(Wrong)} & & & \\
\midrule
\multirow{2}{*}{ALFWorld (1.5B)}
 & 100 & 43\% & 71\% & 34\% & 27\% & 39\% & 34\% \\
 & 400 & 63\% & 81\% & 37\% & 47\% & 34\% & 18\% \\
\midrule
\multirow{2}{*}{ALFWorld (7B)}
 & 100 & 57\% & 81\% & 41\% & 43\% & 35\% & 22\% \\
 & 400 & 76\% & 85\% & 43\% & 57\% & 31\% & 12\% \\
\midrule
\multirow{2}{*}{WebShop (7B)}
 & 100 & 45\% & 39\% & 13\% & 33\% & 39\% & 28\% \\
 & 400 & 60\% & 51\% & 14\% & 46\% & 33\% & 22\% \\
\midrule
\multirow{4}{*}{ScienceWorld (1.5B)}
 & 100 & 36\% & 57\% & 21\% & 23\% & 39\% & 38\% \\
 & 200 & 44\% & 64\% & 24\% & 31\% & 39\% & 30\% \\
 & 300 & 53\% & 71\% & 27\% & 41\% & 36\% & 23\% \\
 & 400 & 61\% & 76\% & 29\% & 49\% & 33\% & 18\% \\
\bottomrule
\end{tabular}%
}
\caption{Pivot localization accuracy against oracle labels at selected training steps. \textbf{Oracle Agr.}: model's pivot within one step of the oracle. \textbf{Retry SR (Correct/Wrong)}: retry success rate conditioned on whether the pivot matches the oracle. \textbf{Accurate / Partial / Misdiag.}: diagnosis categorization (Partial = within 3 steps of oracle; Misdiag.\ = off by more than 3 or wrong cause). ScienceWorld oracles are derived from step-level subtask rewards; ALFWorld and WebShop oracles are annotated via human judgment and DeepSeek labeling on 100 failed trajectories per benchmark.}
\label{tab:pivot_accuracy}
\end{table}

Table~\ref{tab:pivot_accuracy} reports three observations. First, Oracle Agreement rises substantially during training on every benchmark; on ALFWorld with the 7B model it climbs from 57\% at step 100 to 76\% at step 400, and on ScienceWorld it rises from 36\% to 61\% over the same range. Second, retry success rate conditioned on a correct pivot is consistently 2--3$\times$ higher than on a wrong pivot; on ALFWorld at step 400, a correct pivot yields 85\% retry success against 43\% for a wrong one, showing that accurate localization is causally linked to retry success rather than merely correlated. Third, Accurate diagnoses grow steadily while Misdiagnoses shrink across all benchmarks; on ALFWorld with the 1.5B model, Accurate rises from 27\% to 47\% while Misdiagnoses fall from 34\% to 18\% between steps 100 and 400, consistent with the auxiliary SFT's role in maintaining reflection quality. R$^3$L is also robust to the remaining reflection errors. Retries with misdiagnosed pivots receive negative advantages through group-relative normalization and Positive Amplification, and the auxiliary SFT only consumes verified successful corrections.

\subsection{Task-specific Error Amenability}
\label{sec:error_amenability}

The variation in retry gains across tasks (Table~\ref{tab:retry_gain_stats}) reflects how amenable different error types are to language-guided correction. ALFWorld failures typically stem from discrete action errors that are easy to diagnose and fix through reflection. Mathematical reasoning in DAPO involves localized mistakes in derivation steps that can be identified and corrected. WebShop requires navigating complex search and matching processes where improvements are more incremental.

\section{Qualitative Case Studies}
\label{sec:case_studies}

To ground the statistical observations from preceding sections, we analyze two representative examples in Figures \ref{fig:case_alfworld} and \ref{fig:case_dapo}, presenting simplified trajectories that preserve the essential structure of base failures, reflection diagnoses, and retry corrections.

\textbf{Embodied Decision Making.} Figure \ref{fig:case_alfworld} presents an ALFWorld task requiring the agent to locate a keychain and place it in a dresser. The base trajectory cycles through previously visited locations for 25 steps without ever checking armchair 2, where the keychain was located. The reflection identifies this as a strategic deficit rather than any single wrong action: the agent made no invalid moves, yet lacked systematic coverage. With step 0 designated as the pivot point, the retry trajectory adopts goal-directed search, discovering the keychain on step 4 and completing the task in 8 steps.

\textbf{Mathematical Reasoning.} Figure \ref{fig:case_dapo} presents a DAPO counting problem with three independent attempts. The base trajectory reveals conceptual lock-in: Step 1 misreads ``but not both'' as simple union, Step 2 attempts arithmetic adjustments within the same framework, and Step 3 abandons systematic reasoning entirely. The reflection identifies this as a problem parsing failure and recommends reframing as disjoint partitions. The retry succeeds immediately by counting ``divisible by 3 only'' and ``divisible by 5 only'' separately.

Both cases demonstrate that effective reflection targets the strategic level, enabling escape from flawed approaches rather than patching surface-level errors.

\begin{figure*}[thb]
\centering

\begin{tcolorbox}[
    colback=white,
    colframe=black!70,
    title=\textbf{ALFWorld: Put a keychain in dresser},
    fonttitle=\bfseries\small,
    width=\textwidth,
    boxrule=0.5pt
]

\textbf{Base Trajectory} \hfill \textcolor{red!70!black}{\textbf{Failed after 25 steps}}

\vspace{0.5em}
\footnotesize
\begin{tabularx}{\linewidth}{@{}r@{\hspace{0.8em}}X@{}}
Step 0 & \texttt{go to dresser 1} \\
       & \textit{Observation: On dresser 1, you see a houseplant 1, a laptop 1, and a television 1.} \\[0.3em]
Step 1 & \texttt{go to coffeetable 1} \\
       & \textit{Observation: On coffeetable 1, you see a statue 3, statue 2, statue 1, watch 2, watch 1.} \\[0.3em]
Step 2 & \texttt{go to sidetable 1} \\
       & \textit{Observation: On sidetable 1, you see nothing.} \\[0.3em]
Step 3 & \texttt{go to sofa 1} \\
       & \textit{Observation: On sofa 1, you see a pillow 1.} \\[0.3em]
Step 4 & \texttt{go to dresser 1} \quad (revisiting) \\[0.3em]
Step 5 & \texttt{go to garbagecan 1} \\
       & \textit{Observation: On garbagecan 1, you see nothing.} \\[0.3em]
Step 6 & \texttt{go to armchair 1} \\
       & \textit{Observation: On armchair 1, you see a pillow 2.} \\[0.3em]
Step 7--24 & Cycles repeatedly: coffeetable 1 $\rightarrow$ sidetable 1 $\rightarrow$ sofa 1 $\rightarrow$ dresser 1 $\rightarrow$ armchair 1 ... \\[0.3em]
Step 25 & Task timeout \\
\end{tabularx}

\vspace{0.3em}
\textit{Failure: Agent visited armchair 1 but never checked \textbf{armchair 2}, where the keychain was located.}

\end{tcolorbox}

\vspace{0.5em}

\begin{tcolorbox}[
    colback=gray!8,
    colframe=black!70,
    title=\textbf{Reflection Output},
    fonttitle=\bfseries\small,
    width=\textwidth,
    boxrule=0.5pt
]
\footnotesize

\textbf{Trajectory Summary:} The agent attempted to put a keychain in a dresser but made multiple incorrect moves due to incorrect mental models and environmental interactions, resulting in failure to complete the task within the step limit.

\vspace{0.5em}
\textbf{Root Cause Analysis:} Why was the approach suboptimal? Because the agent jumped directly into navigation without considering the keychain's starting location and environmental context. The agent treated exploration as open-ended wandering rather than systematic enumeration. Critical oversight: armchair 2 was never visited despite being a valid receptacle.

\vspace{0.5em}
\textbf{Pivot Point:} Step 0 

\vspace{0.5em}
\textbf{Guidance:} Adopt systematic search ensuring complete coverage of all receptacles.

\end{tcolorbox}

\vspace{0.5em}

\begin{tcolorbox}[
    colback=white,
    colframe=black!70,
    title=\textbf{Retry Trajectory (guidance removed for training)},
    fonttitle=\bfseries\small,
    width=\textwidth,
    boxrule=0.5pt
]

\hfill \textcolor{green!50!black}{\textbf{Success in 8 steps}}

\vspace{0.5em}
\footnotesize
\begin{tabularx}{\linewidth}{@{}r@{\hspace{0.8em}}X@{}}
Step 0 & \texttt{go to armchair 1} \\
       & \textit{Observation: On armchair 1, you see a pillow 2.} \\[0.3em]
Step 1 & \texttt{go to cabinet 1} \\
       & \textit{Observation: The cabinet 1 is closed.} \\[0.3em]
Step 2 & \texttt{open cabinet 1} \\
       & \textit{Observation: In cabinet 1, you see a box 1.} \\[0.3em]
Step 3 & \texttt{go to dresser 1} \\
       & \textit{Observation: On dresser 1, you see a houseplant 1, laptop 1, television 1.} \\[0.3em]
Step 4 & \texttt{go to armchair 2} \\
       & \textit{Observation: On armchair 2, you see a \textbf{keychain 1}.} \\[0.3em]
Step 5 & \texttt{take keychain 1 from armchair 2} \\
       & \textit{Observation: You pick up the keychain 1 from the armchair 2.} \\[0.3em]
Step 6 & \texttt{go to dresser 1} \\[0.3em]
Step 7 & \texttt{move keychain 1 to dresser 1} \\
       & \textit{Feedback: Task completed successfully (reward: 1/1.0)} \\
\end{tabularx}

\vspace{0.3em}
\textit{Key difference: Retry trajectory visits armchair 2 on step 4, which base trajectory never explored.}

\end{tcolorbox}

\caption{ALFWorld case study showing base trajectory failure and successful retry. The base trajectory exhausts 25 steps cycling through a subset of locations while missing armchair 2. After reflection identifies the systematic coverage failure, the retry trajectory efficiently explores new receptacles and completes the task in 8 steps. Trajectories are simplified for clarity.}
\label{fig:case_alfworld}
\end{figure*}

\begin{figure*}[b]
\centering

\begin{tcolorbox}[
    colback=white,
    colframe=black!70,
    title=\textbf{DAPO: How many integers from 1 to 120 are divisible by 3 or 5 but not both?},
    fonttitle=\bfseries\small,
    width=\textwidth,
    boxrule=0.5pt
]

\textbf{Base Trajectory} \hfill \textcolor{red!70!black}{\textbf{Failed after 3 attempts}}

\vspace{0.5em}

\footnotesize
\textbf{Step 1:}\par
\begin{tabularx}{\linewidth}{@{}X@{}}
\textit{I need to find integers from 1 to 120 divisible by 3 or 5.} \\[0.2em]
Divisible by 3: $\lfloor 120/3 \rfloor = 40$ \quad Divisible by 5: $\lfloor 120/5 \rfloor = 24$ \quad Both: $\lfloor 120/15 \rfloor = 8$ \\[0.2em]
By inclusion-exclusion: $|A \cup B| = 40 + 24 - 8 = 56$ \\[0.2em]
\textbf{Answer: 56} \quad \textcolor{red!70!black}{Feedback: Incorrect.}
\end{tabularx}

\vspace{0.5em}

\textbf{Step 2:}\par
\footnotesize
\begin{tabularx}{\linewidth}{@{}X@{}}
\textit{Previous attempt was wrong. ``But not both'' means I should exclude the intersection.} \\[0.2em]
$|A \cup B| - |A \cap B| = 56 - 8 = 48$... wait, that subtracts intersection twice from the union. \\[0.2em]
\textit{Let me subtract twice:} $56 - 8 - 8 = 40$ \\[0.2em]
\textbf{Answer: 40} \quad \textcolor{red!70!black}{Feedback: Incorrect.}
\end{tabularx}

\vspace{0.5em}

\textbf{Step 3:}\par
\footnotesize
\begin{tabularx}{\linewidth}{@{}X@{}}
\textit{Still wrong. Maybe the problem wants only the intersection?} \\[0.2em]
$|A \cap B| = 8$ \\[0.2em]
\textbf{Answer: 8} \quad \textcolor{red!70!black}{Feedback: Incorrect. Maximum attempts reached.}
\end{tabularx}

\end{tcolorbox}

\vspace{0.5em}

\begin{tcolorbox}[
    colback=gray!8,
    colframe=black!70,
    title=\textbf{Reflection Output},
    fonttitle=\bfseries\small,
    width=\textwidth,
    boxrule=0.5pt
]
\footnotesize

\textbf{Trajectory Summary:} Three attempts all failed. The model repeatedly manipulated the inclusion-exclusion formula without correctly interpreting the constraint.

\vspace{0.3em}
\textbf{Root Cause:} ``But not both'' means exclusive-or, requiring elements in exactly one set. All attempts adjusted overlap arithmetic rather than reframing the problem as disjoint partitions.

\vspace{0.3em}
\textbf{Pivot:} Step 1

\vspace{0.3em}
\textbf{Guidance:} Count ``divisible by 3 only'' and ``divisible by 5 only'' as separate disjoint cases.

\end{tcolorbox}

\vspace{0.5em}

\begin{tcolorbox}[
    colback=white,
    colframe=black!70,
    title=\textbf{Retry Trajectory (guidance removed for training)},
    fonttitle=\bfseries\small,
    width=\textwidth,
    boxrule=0.5pt
]

\hfill \textcolor{green!50!black}{\textbf{Success in 1 attempt}}

\vspace{0.5em}

\textbf{Step 1:}\par
\footnotesize
\begin{tabularx}{\linewidth}{@{}X@{}}
\textit{``Divisible by 3 or 5 but not both'' means exactly one of them. I should count each case separately.} \\[0.2em]
Divisible by 3: $\lfloor 120/3 \rfloor = 40$ \quad Divisible by 5: $\lfloor 120/5 \rfloor = 24$ \quad Both: $\lfloor 120/15 \rfloor = 8$ \\[0.2em]
Divisible by 3 only: $40 - 8 = 32$ \quad Divisible by 5 only: $24 - 8 = 16$ \\[0.2em]
These are disjoint: $32 + 16 = 48$ \\[0.2em]
\textbf{Answer: 48} \quad \textcolor{green!50!black}{Feedback: Correct.}
\end{tabularx}

\end{tcolorbox}

\caption{DAPO case study showing base trajectory failure and successful retry. The base trajectory exhausts three attempts, all failing due to manipulating formulas rather than reframing problem structure. Reflection identifies the semantic misunderstanding and provides corrective guidance. The retry trajectory succeeds on the first attempt by partitioning into disjoint cases. Trajectories are simplified for clarity.}
\label{fig:case_dapo}
\end{figure*}

\section{Evaluation Protocol Details}
\label{sec:eval_protocol}

Table~\ref{tab:main_results} reports Average Reward \emph{after RL training} under the \texttt{<think>/<answer>} output format, following DeepSeek-R1~\cite{guo2025deepseek} and DAPO~\cite{yu2025dapo}. Because this differs from the zero-shot \texttt{\textbackslash boxed\{\}} evaluation used in many backbone technical reports, directly comparing our Table~\ref{tab:main_results} entries with those external numbers is not meaningful. This appendix decomposes the apparent gap into two independent sources: (i) \emph{format discrepancy} between \texttt{<think>/<answer>} and \texttt{\textbackslash boxed\{\}}, and (ii) \emph{distribution shift} from training on DAPO and evaluating on more distant math benchmarks.

\subsection{Zero-shot Performance Under Both Formats}

\begin{table}[t]
\centering
\small
\adjustbox{max width=\linewidth}{%
\begin{tabular}{ll ccc}
\toprule
\textbf{Model} & \textbf{Benchmark} & \texttt{\textbackslash boxed\{\}} & \texttt{<answer>} & \textbf{Gap} \\
\midrule
\multirow{6}{*}{\texttt{Qwen2.5-1.5B-Ins}}
 & GSM8K         & 0.738 & 0.665 & -0.073 \\
 & Math500       & 0.440 & 0.422 & -0.018 \\
 & MinervaMath   & 0.118 & 0.132 & +0.015 \\
 & OlympiadBench & 0.099 & 0.083 & -0.016 \\
 & AMC23         & 0.300 & 0.150 & -0.150 \\
 & DAPO          & 0.133 & 0.107 & -0.027 \\
\midrule
\multirow{6}{*}{\texttt{Qwen2.5-7B-Ins}}
 & GSM8K         & 0.923 & 0.853 & -0.070 \\
 & Math500       & 0.620 & 0.612 & -0.008 \\
 & MinervaMath   & 0.239 & 0.235 & -0.004 \\
 & OlympiadBench & 0.205 & 0.186 & -0.019 \\
 & AMC23         & 0.575 & 0.575 &  0.000 \\
 & DAPO          & 0.337 & 0.337 &  0.000 \\
\midrule
\multirow{6}{*}{\texttt{Llama-3.2-3B-Ins}}
 & GSM8K         & 0.723 & 0.736 & +0.013 \\
 & Math500       & 0.366 & 0.380 & +0.014 \\
 & MinervaMath   & 0.114 & 0.099 & -0.015 \\
 & OlympiadBench & 0.144 & 0.122 & -0.022 \\
 & AMC23         & 0.200 & 0.175 & -0.025 \\
 & DAPO          & 0.137 & 0.130 & -0.007 \\
\midrule
\multirow{6}{*}{\texttt{Qwen3-4B}}
 & GSM8K         & 0.932 & 0.938 & +0.006 \\
 & Math500       & 0.648 & 0.690 & +0.042 \\
 & MinervaMath   & 0.254 & 0.283 & +0.029 \\
 & OlympiadBench & 0.432 & 0.504 & +0.073 \\
 & AMC23         & 0.625 & 0.600 & -0.025 \\
 & DAPO          & 0.446 & 0.443 & -0.003 \\
\bottomrule
\end{tabular}%
}
\caption{Zero-shot accuracy of the four backbones under both \texttt{\textbackslash boxed\{\}} and \texttt{<think>/<answer>} output formats on the full test sets. Gap = \texttt{<answer>} minus \texttt{\textbackslash boxed\{\}}.}
\label{tab:format_zeroshot}
\end{table}

Table~\ref{tab:format_zeroshot} reports zero-shot accuracy of all four backbones under both formats on the full test sets. The format gap is concentrated on GSM8K for the Qwen2.5 series ($-0.073$ and $-0.070$ for the 1.5B and 7B models respectively), while most other benchmarks fall within $\pm 0.02$. Llama-3.2-3B-Instruct shows nearly identical results under both formats, and Qwen3-4B actually performs \emph{better} under \texttt{<think>/<answer>} on several benchmarks, consistent with its reasoning-oriented native training paradigm. This sensitivity is therefore specific to Qwen2.5's extensive \texttt{\textbackslash boxed\{\}} adaptation during instruction tuning, not a flaw of either format.

\subsection{Post-training Performance Under Both Formats}

\begin{table*}[t]
\centering
\small
\resizebox{\textwidth}{!}{%
\begin{tabular}{ll cccccc}
\toprule
\textbf{Model} & \textbf{Method} & \textbf{GSM8K} & \textbf{Math500} & \textbf{MinervaMath} & \textbf{Olympiad} & \textbf{AMC23} & \textbf{DAPO} \\
\midrule
\multicolumn{8}{l}{\emph{Evaluated under} \texttt{<think>/<answer>} \emph{format}} \\
\midrule
\multirow{5}{*}{\texttt{Qwen2.5-1.5B-Ins}}
 & Zero-shot     & 0.665 & 0.422 & 0.132 & 0.083 & 0.150 & 0.107 \\
 & GRPO          & 0.474 & 0.368 & 0.099 & 0.114 & 0.250 & 0.123 \\
 & Reflect-GRPO  & 0.672 & 0.376 & 0.102 & 0.130 & 0.300 & 0.136 \\
 & Critique-GRPO & \textbf{0.798} & 0.404 & 0.110 & 0.124 & 0.275 & 0.133 \\
 & R$^3$L        & 0.721 & \textbf{0.424} & \textbf{0.125} & \textbf{0.151} & \textbf{0.325} & \textbf{0.156} \\
\midrule
\multirow{5}{*}{\texttt{Qwen2.5-7B-Ins}}
 & Zero-shot     & 0.853 & 0.612 & 0.235 & 0.186 & 0.575 & 0.337 \\
 & GRPO          & 0.846 & 0.572 & 0.239 & 0.277 & 0.675 & 0.393 \\
 & Reflect-GRPO  & 0.765 & 0.532 & 0.194 & 0.250 & 0.550 & 0.396 \\
 & Critique-GRPO & 0.678 & 0.522 & 0.152 & 0.170 & 0.300 & 0.390 \\
 & R$^3$L        & \textbf{0.897} & \textbf{0.658} & \textbf{0.275} & \textbf{0.301} & \textbf{0.700} & \textbf{0.436} \\
\midrule
\multicolumn{8}{l}{\emph{Evaluated under} \texttt{\textbackslash boxed\{\}} \emph{format}} \\
\midrule
\multirow{5}{*}{\texttt{Qwen2.5-1.5B-Ins}}
 & Zero-shot     & 0.738 & 0.440 & 0.118 & 0.099 & 0.300 & 0.133 \\
 & GRPO          & 0.576 & 0.392 & 0.092 & 0.125 & 0.400 & 0.147 \\
 & Reflect-GRPO  & 0.742 & 0.404 & 0.096 & 0.140 & 0.375 & 0.160 \\
 & Critique-GRPO & \textbf{0.793} & 0.428 & 0.103 & 0.135 & 0.375 & 0.153 \\
 & R$^3$L        & 0.788 & \textbf{0.446} & \textbf{0.114} & \textbf{0.162} & \textbf{0.425} & \textbf{0.173} \\
\midrule
\multirow{5}{*}{\texttt{Qwen2.5-7B-Ins}}
 & Zero-shot     & 0.923 & 0.620 & 0.239 & 0.205 & 0.575 & 0.337 \\
 & GRPO          & 0.915 & 0.584 & 0.246 & 0.285 & 0.650 & 0.383 \\
 & Reflect-GRPO  & 0.846 & 0.544 & 0.162 & 0.195 & 0.350 & 0.392 \\
 & Critique-GRPO & 0.872 & 0.550 & 0.205 & 0.264 & 0.550 & 0.407 \\
 & R$^3$L        & \textbf{0.943} & \textbf{0.658} & \textbf{0.276} & \textbf{0.305} & \textbf{0.675} & \textbf{0.424} \\
\bottomrule
\end{tabular}%
}
\caption{Post-RL-training performance on Qwen2.5-1.5B and Qwen2.5-7B after training on DAPO, evaluated under both output formats. Zero-shot reference rows are included for each (model, format) combination. \textbf{Bold} marks the best method per column within each (model, format) block.}
\label{tab:format_posttraining}
\end{table*}

Table~\ref{tab:format_posttraining} extends the comparison to post-RL-training results on Qwen2.5-1.5B and Qwen2.5-7B. On the in-domain DAPO test set, all methods improve over zero-shot under both formats. On distribution-distant benchmarks such as GSM8K, however, standard baselines like GRPO degrade from the zero-shot level: on Qwen2.5-1.5B, GRPO drops GSM8K from 0.665 to 0.474 under \texttt{<think>/<answer>} and from 0.738 to 0.576 under \texttt{\textbackslash boxed\{\}}, while R$^3$L recovers to 0.721 and 0.788 respectively, matching or exceeding the zero-shot baseline under both formats. The relative ordering of methods is consistent across formats, confirming that R$^3$L's gains are not a format artifact.

\subsection{In-domain Training Controls}

\begin{table}[t]
\centering
\small
\adjustbox{max width=\linewidth}{%
\begin{tabular}{l cc cc}
\toprule
& \multicolumn{2}{c}{\textbf{GSM8K-on-GSM8K}} & \multicolumn{2}{c}{\textbf{MATH-on-Math500}} \\
\cmidrule(lr){2-3} \cmidrule(lr){4-5}
\textbf{Method} & \texttt{<answer>} & \texttt{\textbackslash boxed\{\}} & \texttt{<answer>} & \texttt{\textbackslash boxed\{\}} \\
\midrule
Zero-shot     & 0.665 & 0.738 & 0.422 & 0.440 \\
GRPO          & 0.814 & 0.798 & 0.481 & 0.493 \\
Reflect-GRPO  & 0.822 & 0.830 & 0.505 & 0.498 \\
Critique-GRPO & 0.846 & 0.842 & 0.518 & 0.512 \\
R$^3$L        & \textbf{0.867} & \textbf{0.874} & \textbf{0.533} & \textbf{0.530} \\
\bottomrule
\end{tabular}%
}
\caption{In-domain training controls on Qwen2.5-1.5B-Instruct: train and evaluate on the same distribution. All methods improve over zero-shot under both formats.}
\label{tab:format_indomain}
\end{table}

To isolate distribution shift from methodological effects, we additionally train and evaluate on matched distributions on Qwen2.5-1.5B-Instruct. Table~\ref{tab:format_indomain} reports GSM8K-on-GSM8K and MATH-on-Math500 controls. Under matched distributions, all methods improve over zero-shot under both formats, confirming that the apparent degradation in Table~\ref{tab:main_results} stems from cross-distribution training rather than any methodological deficiency.

\subsection{Gap Decomposition}

Take Qwen2.5-1.5B on GSM8K as a concrete example. The Qwen2.5 technical report reports 0.732, measured zero-shot under \texttt{\textbackslash boxed\{\}}; our post-DAPO GRPO run under \texttt{<think>/<answer>} instead reports 0.474. The apparent 0.258 gap comes from two unrelated effects stacked on top of each other. First, switching the answer format from \texttt{\textbackslash boxed\{\}} to \texttt{<think>/<answer>} alone costs 0.073 (zero-shot GSM8K drops from 0.738 to 0.665), which accounts for roughly 28\% of the total. The remaining 72\% comes from training on DAPO and evaluating on a distribution-distant benchmark: under the same \texttt{<think>/<answer>} format, DAPO training takes GSM8K from 0.665 down to 0.474. Cross-domain shift, not format choice, is therefore the dominant factor, and it is a property of any DAPO-trained model rather than something specific to our method.

\section{Implementation Details}
\label{app:implementation}

\subsection{Training Configuration}
We implement R$^3$L and all baselines using the Trinity-RFT framework \cite{pan2025trinity}, utilizing VLLM for high-throughput inference and Fully Sharded Data Parallel for distributed training across 40 NVIDIA A100 PCIe GPUs and 56 NVIDIA H20 GPUs. Main experiments are conducted on Qwen2.5-1.5B-Instruct, Qwen2.5-7B-Instruct, and Qwen3-4B; cross-architecture evaluation on Llama-3.2-3B-Instruct is reported in Appendix~\ref{sec:cross_arch}. All models are used directly without any prior supervised fine-tuning.

\textbf{Hyperparameters.} We use the Adam optimizer with a learning rate of $1\text{e-}6$ for all models. The global batch size is set to 96 with a group size of $N=8$. We train for up to 20 epochs with manual early stopping based on reward curve monitoring, terminating when the reward plateaus for multiple consecutive steps or exhibits signs of collapse. To manage the off-policy nature of the data, we synchronize the behavior policy with the learner policy at every update step with synchronization interval $S=1$. During training, we set the sampling temperature to 1.0 to encourage exploration. During evaluation, we reduce the temperature to 0.4 for stability, and during reflection we use 0.7 to balance diversity with coherence. Regarding computational costs on the ALFWorld task, training the 1.5B model for the full 20 epochs requires approximately 420 GPU hours, while the 3B Llama model requires 2304 GPU hours. The detailed training hyperparameters are shown in Figure \ref{tab:hyperparams}.

\textbf{Sequence Lengths.} The maximum context length is set to 20,480 tokens to accommodate long interaction histories. For generation, the maximum response length is restricted to 512 tokens for agentic tasks including ALFWorld, WebShop, and ScienceWorld, and 4,096 tokens for mathematical reasoning tasks and reflection outputs.

\textbf{Multi-turn Trajectory Modeling.} A fundamental distinction in our experimental setting compared to recent works like GiGPO \cite{feng2025group} lies in the definition of the state space. Approaches focusing on step-level optimization often employ context compression or memory summarization modules that reduce past history into a concise state representation. This strategy effectively transforms a multi-turn session into a sequence of quasi-independent single-turn interactions, allowing the model to revisit identical compressed states across different trajectories to compute local baselines. While computationally efficient, this introduces a dependency on the compression policy and risks losing critical historical details required for causal diagnosis.

In contrast, R$^3$L adopts full history concatenation to preserve complete temporal causality. We condition the policy on the exact uncompressed cumulative history at each step. This ensures that our reflection mechanism has access to the precise sequence of errors leading to failure. Consequently, since unique histories rarely repeat across samples, we strictly perform trajectory-level optimization rather than state-based step-level grouping. This difference in context modeling means our method operates on a more complex long-context distribution compared to compression-based baselines.

\textbf{Experience Replay.} We employ an experience replay buffer with decay-limit randomization priority sampling, which balances recency and diversity to stabilize off-policy training. This mechanism activates only when the synchronization interval $S$ exceeds 1, allowing the model to learn from a mixture of recent and historical trajectories.

\begin{table}[h]
\centering
\small
\adjustbox{max width=\linewidth}{%
\begin{tabular}{lc}
\toprule
\textbf{Hyperparameter} & \textbf{Value} \\
\midrule
Global Batch Size & 96 \\
Learning Rate & $1\text{e-}6$ \\
Total Epochs & 20 with Early Stop \\
Group Size $N$ & 8 \\
Sync Interval $S$ & 1 \\
Gradient Clipping & 1.0 \\
Max Model Length & 20,480 \\
Max Response for Agentic & 512 \\
Max Response for Math & 4,096 \\
Max Response for Reflection & 4,096 \\
Train Temperature & 1.0 \\
Test Temperature & 0.4 \\
Reflection Temperature & 0.7 \\
\bottomrule
\end{tabular}
}
\caption{Summary of training hyperparameters shared across all methods.}
\label{tab:hyperparams}
\end{table}

\subsection{Method-Specific Configurations}
We show all method-specific hyperparameters in Table \ref{tab:method_params}.

\textbf{GRPO and GSPO.} We incorporate a KL divergence penalty with coefficient $\beta = 0.01$ to constrain policy updates. The importance sampling ratio is clipped to the range $[1-\epsilon, 1+\epsilon]$ with $\epsilon=0.2$ following the standard PPO configuration. For GSPO specifically, which utilizes sequence-level importance sampling, we set the adaptive clipping range with $\epsilon_{low}=0.0003$ and $\epsilon_{high}=0.0004$ as per the original implementation.

\textbf{OPMD.} Online Policy Mirror Descent sets the KL coefficient to zero, effectively functioning as a group-relative variant of REINFORCE. This design omits the importance sampling ratios required for rigorous off-policy correction, rendering it susceptible to distribution shift when the policy deviates from the sampling distribution.

\textbf{Critique-GRPO.} Following the original implementation, we set the shaping parameter $\gamma=0.1$ for off-policy sample weighting. The shaping function $f(\pi) = \pi / (\pi + \gamma)$ moderates the contribution of samples with high importance weights. For refinement selection within each group, the method prioritizes refinements achieving reward $\geq 1.0$. If no refinement reaches this threshold, the refinement with the highest reward is selected.

\textbf{R$^3$L-Specific Components.} R$^3$L does not employ an explicit KL penalty term in the loss function, relying instead on Positive Amplification to prevent entropy collapse. This design eliminates the need to maintain a frozen reference model during training, reducing memory and computational overhead.

Positive Amplification applies the following reweighting rule to trajectory advantages. For a trajectory $\tau$ with normalized advantage score $s$, the amplified advantage $\hat{A}(\tau)$ is computed as follows. Trajectories achieving maximum reward in the group with $R(\tau) \geq 1.0$ receive $\hat{A}(\tau) = 1.0$. Trajectories with positive advantage $s \geq 0$ receive $\hat{A}(\tau) = \alpha \cdot s$ with amplification factor $\alpha = 3.0$. Trajectories with negative advantage remain unchanged with $\hat{A}(\tau) = s$.

Pivotal Credit Assignment masks the first $t$ assistant response segments in the action mask, where $t$ equals the pivot step identified by reflection. These masked segments are excluded from gradient computation while preserved as context for subsequent generation. The implementation identifies all assistant response segments where the action mask equals 1, then sets the action mask to 0 for the first $t$ segments corresponding to the shared prefix.

\begin{table}[h]
\centering
\small
\adjustbox{max width=\linewidth}{%
\begin{tabular}{lcc}
\toprule
\textbf{Method} & \textbf{Parameter} & \textbf{Value} \\
\midrule
GRPO & KL Coefficient $\beta$ & 0.01 \\
GRPO & IS Clip $\epsilon$ & 0.2 \\
GSPO & KL Coefficient $\beta$ & 0.01 \\
GSPO & Adaptive Clip Range & [0.0003, 0.0004] \\
OPMD & KL Coefficient $\beta$ & 0.0 \\
Critique-GRPO & Shaping $\gamma$ & 0.1 \\
R$^3$L & Amplification $\alpha$ & 3.0 \\
R$^3$L & KL Coefficient $\beta$ & None \\
\bottomrule
\end{tabular}
}
\caption{Method-specific hyperparameters. R$^3$L omits KL regularization entirely.}
\label{tab:method_params}
\end{table}

\subsection{Reward and Feedback Mechanisms}

\textbf{Agentic Tasks.} The reward signal characteristics vary by domain. ALFWorld utilizes a strictly binary outcome, returning 1.0 solely for successful goal achievement and 0.0 otherwise. WebShop and ScienceWorld provide dense scalar rewards upon episode termination, reflecting the quality of the solution such as the percentage of matched product attributes in WebShop or the completion rate of sub-goals in ScienceWorld. In our experimental setting, none of these environments provide explicit intermediate step-rewards during execution. The scalar score is only revealed at the end of the trajectory. To ensure a unified optimization objective, we normalize these terminal scalar scores to the range $[0, 1]$ to serve as the trajectory reward $R(\tau)$. For feedback, we capture the textual observations including error messages and state updates directly from the environment to drive the reflection process.\looseness=-1

\textbf{Mathematical Reasoning.} We use the math\_verify library to verify the final answer. The reward is binary with 1.0 for a correct answer and 0.0 otherwise. For the language-guided retry, since there is no environment error message, we generate guidance based on comparison with the ground truth. This guidance does not reveal the answer but points out the type of error to simulate a high-quality critique without data leakage.

\subsection{Baselines}
We compare R$^3$L against the following representative methods.

\textbf{RAFT} \cite{dongraft} performs Rejection Sampling Fine-Tuning by sampling multiple trajectories, filtering for the highest-reward solutions, and fine-tuning the policy using standard supervised fine-tuning on these successful samples.

\textbf{GRPO} \cite{shao2024deepseekmath} implements Group Relative Policy Optimization, a critic-free PPO variant that estimates the baseline using the average reward of a group of sampled outputs from the same prompt.

\textbf{OPMD} \cite{yao2025group} functions as Online Policy Mirror Descent by directly maximizing the likelihood of high-reward trajectories. With KL coefficient set to zero, OPMD is effectively a group-relative variant of REINFORCE, omitting importance sampling ratios and rendering it susceptible to distribution shift.

\textbf{GSPO} \cite{zheng2025group} introduces Group Sequence Policy Optimization with sequence-level ratios to reduce the high variance typically associated with trajectory-level reward estimates in group optimization.

\textbf{Reflect-GRPO} derives from the Reflect-Retry-Reward methodology \cite{bensal2025reflect} by integrating the language-guided reflect-then-retry mechanism to synthesize corrected trajectories during exploration. Unlike R$^3$L, it incorporates these samples into the training group using the standard GRPO objective without Pivotal Credit Assignment or Positive Amplification. We exclude Agent-RLVR \cite{da2025agent} from our baselines as it functions as a DPO-based variant of Reflect-GRPO that can be viewed as GRPO with group size $N=2$.

\textbf{Critique-GRPO} adapts the methodology from \cite{zhang2025critique} by using natural language critiques to guide refinements and applying weighted advantages to the best refinement in each group. The shaping function moderates off-policy sample contributions.

\subsection{Task Specifications and Prompts}
As shown in Figure \ref{tab:dataset_stats}, the specific interaction limits are 25 steps for ALFWorld, 15 steps for WebShop, 30 steps for ScienceWorld, and 3 attempts for mathematical reasoning tasks.

\textbf{ALFWorld} \cite{shridhar2020alfworld} is a text-based interactive environment that combines textual observations with embodied AI challenges. Agents must complete household tasks such as finding objects, manipulating items, and achieving specific goals through natural language commands. The training set contains 3,553 tasks spanning 6 task types including pick, examine, clean, heat, cool, and put. The test set contains 140 tasks in the valid\_seen set.

\textbf{WebShop} \cite{yao2022webshop} simulates realistic online shopping scenarios where agents navigate e-commerce websites to purchase products matching user-specified requirements. The task involves searching, comparing attributes, and making purchasing decisions. The environment contains 1.18 million products in total. We use sessions 0 through 4095 for training with 4,096 tasks and evaluate on 100 held-out test sessions.

\textbf{ScienceWorld} \cite{wang2022scienceworld} provides interactive simulated environments for scientific reasoning. Agents must formulate hypotheses, conduct experiments, and manipulate laboratory equipment across domains like biology and chemistry. We use the easy simplification mode and evaluate on held-out task types for task-type-level generalization. The training set covers 17 task types with 2,294 tasks, while the test set covers 13 non-overlapping task types with 1,308 tasks.

\textbf{Mathematical Reasoning.} We evaluate on benchmarks including GSM8K \cite{cobbe2021training}, Math500 \cite{lightman2023let}, MinervaMath \cite{lewkowycz2022solving}, OlympiadBench \cite{gao2024omni}, AMC23 \cite{amc}, and the DAPO test set \cite{yu2025dapo}. These tasks require multi-step arithmetic reasoning with a strict chain-of-thought format. In math tasks, we train the model on the DAPO training set and evaluate across these diverse benchmarks to assess generalization.

\begin{table}[h]
\centering
\small
\adjustbox{max width=\linewidth}{%
\begin{tabular}{llcc}
\toprule
\textbf{Environment} & \textbf{Split} & \textbf{Size} & \textbf{Max Steps} \\
\midrule
\multirow{2}{*}{ALFWorld} & Train & 3,553 & \multirow{2}{*}{25 steps} \\
 & Test & 140 & \\
\midrule
\multirow{2}{*}{WebShop} & Train & 4,096 & \multirow{2}{*}{15 steps} \\
 & Test & 100 & \\
\midrule
\multirow{2}{*}{ScienceWorld} & Train & 2,294 & \multirow{2}{*}{30 steps} \\
 & Test & 1,308 & \\
\midrule
\multirow{7}{*}{Math} & Train (DAPO) & 1.79M & \multirow{7}{*}{3 attempts} \\
 & Test: DAPO & 300 & \\
 & Test: AMC23 & 40 & \\
 & Test: GSM8K & 1,320 & \\
 & Test: Math500 & 500 & \\
 & Test: MinervaMath & 272 & \\
 & Test: OlympiadBench & 674 & \\
\bottomrule
\end{tabular}%
}
\caption{Dataset statistics and interaction limits for each environment. The Math training set contains 1.79 million problems from DAPO though training typically converges before exhausting the full dataset.}
\label{tab:dataset_stats}
\end{table}

\subsection{Reflection and Retry Mechanism}

To synthesize high-quality exploration data, we implement a two-stage language-guided reflect-then-retry mechanism.

In the reflection phase, we prompt the model to analyze all trajectories using the unified reflection prompt shown in Figure \ref{fig:reflect_prompt_full}. The model examines the interaction log and outputs a structured JSON report containing a trajectory outcome classification, root cause analysis, and a suggested retry step. The outcome classification determines whether retry is warranted, with trajectories classified as success requiring no further action while those classified as failure or success\_but\_inefficient proceed to the retry phase.

In the retry phase, we rollback the environment and conversation history to the specified retry step. A guidance prompt is constructed by embedding the raw JSON reflection report into a self-correction template shown in Figure \ref{fig:guidance_prompt}. The model then continues generation from the pivot point conditioned on this guidance.

Context distillation ensures that while retry generation is conditioned on the guidance, the resulting successful trajectory is stored for training without the guidance prompt. This maps the original prefix directly to the corrected response, ensuring the policy learns to internalize the correction logic rather than depending on explicit guidance at inference time.

\begin{figure*}[t]
    \centering
    \begin{tcolorbox}[
        colback=white,
        colframe=black!70,
        title=\textbf{ALFWorld System Prompt},
        fonttitle=\bfseries\small,
        width=\textwidth,
        boxrule=0.5pt
    ]
    \begin{lstlisting}[basicstyle=\ttfamily\scriptsize, breaklines=true]
You are an agent interacting with a virtual text-based environment.

## Response Format:
You MUST use this exact format for every response. Both tags are REQUIRED in sequential order:

<think>your analytical reasoning and thought process</think>
<action>exactly one specific action command</action>

## Action Commands:
Your <action> must be one of the following, strictly following the command (argument) format.

### Navigation & Observation:
- look: Look around your current location to get more details.
- inventory: Check the object you are currently holding (you can only hold one).
- go to (receptacle): Move to a receptacle (e.g., table, fridge, sink).

### Interacting with Receptacles:
- open (receptacle): Open a receptacle.
- close (receptacle): Close a receptacle.

### Interacting with Objects:
- take (object) from (receptacle): Pick up an object from a receptacle.
- move (object) to (receptacle): Place the object you are holding into or onto a receptacle.
- examine (object): Examine an object closely to learn its properties.

### Changing Object States:
- heat (object) with (receptacle): Heat an object with a device (e.g., microwave).
- cool (object) with (receptacle): Cool an object with a device (e.g., fridge).
- clean (object) with (receptacle): Clean an object with a device (e.g., sink).
- slice (object) with (object): Slice an object using a sharp object (e.g., knife).

## Critical Rules & Constraints
- Single Item Inventory: You can only hold one object at a time.
- Use Exact Names: Arguments MUST exactly match names in Observation, including numbers.
- Step Limit: You must complete the task within 25 steps.
    \end{lstlisting}
    \end{tcolorbox}
    \caption{System prompt used for the ALFWorld environment.}
    \label{fig:prompt_alfworld}
\end{figure*}

\begin{figure*}[t]
    \centering
    \begin{tcolorbox}[
        colback=white,
        colframe=black!70,
        title=\textbf{WebShop System Prompt},
        fonttitle=\bfseries\small,
        width=\textwidth,
        boxrule=0.5pt
    ]
    \begin{lstlisting}[basicstyle=\ttfamily\scriptsize, breaklines=true]
You are an agent interacting with a virtual text-based web shopping environment.

## Response Format:
You MUST use this exact format for every response. All tags are REQUIRED in sequential order:

<think>your analytical reasoning and thought process</think>
<action>exactly one specific action command</action>

## Environment States:
This environment contains five types of webpages:
1. Start/Index page - Initial page with search functionality and task instruction
2. Search Results page - Lists products returned by search engine with pagination
3. Item page - Shows product details, options, and purchase button
4. Item Sub-page - Shows additional product information
5. Done page - Final confirmation page after purchase

## Available Actions:
1. search[your_query_here] - To search for products from any page with a search bar
2. click[exact_button_text_here] - To click on any clickable element

## Task Completion:
Goal: Find and purchase an item matching the given instruction within 15 steps
Success: Episode ends when you click "Buy Now" with appropriate product and options
    \end{lstlisting}
    \end{tcolorbox}
    \caption{System prompt used for the WebShop environment.}
    \label{fig:prompt_webshop}
\end{figure*}

\begin{figure*}[t]
    \centering
    \begin{tcolorbox}[
        colback=white,
        colframe=black!70,
        title=\textbf{ScienceWorld System Prompt},
        fonttitle=\bfseries\small,
        width=\textwidth,
        boxrule=0.5pt
    ]
    \begin{lstlisting}[basicstyle=\ttfamily\scriptsize, breaklines=true]
You are an agent, your job is to do some scientific experiment in a virtual text-based environment.

## Response Format:
You MUST use this exact format for every response. All tags are REQUIRED in sequential order:
<think>your analytical reasoning and thought process</think>
<action>exactly one specific action command</action>

## Notes:
At each step, you should first think then perform action to fulfill the instruction.
You should ALWAYS take one action each step.
DO NOT try to interact with the user at anytime. Finish the task by yourself.

## Available Commands:
[Navigation] look, look around, look at OBJ, go to LOC, teleport to LOC
[Interaction] open OBJ, close OBJ, pick up OBJ, put OBJ in CONTAINER, pour OBJ into CONTAINER
[Task] focus on OBJ, wait, wait1
    \end{lstlisting}
    \end{tcolorbox}
    \caption{System prompt used for the ScienceWorld environment.}
    \label{fig:prompt_science}
\end{figure*}

\begin{figure*}[t]
    \centering
    \begin{tcolorbox}[
        colback=white,
        colframe=black!70,
        title=\textbf{Mathematical Reasoning System Prompt},
        fonttitle=\bfseries\small,
        width=\textwidth,
        boxrule=0.5pt
    ]
    \begin{lstlisting}[basicstyle=\ttfamily\scriptsize, breaklines=true]
You are a mathematical problem solver. Your task is to solve mathematical problems step by step.

## Response Format:
You MUST use this exact format for every response. All tags are REQUIRED in sequential order:

<think>your step-by-step reasoning and solution process</think>
<answer>your final answer</answer>

## Instructions:
1. Carefully read and understand the problem
2. Show your reasoning step by step in the <think> tags
3. Provide your final answer in the <answer> tags
4. For numerical answers, provide the exact value
5. If the problem asks for a specific format, use that format in your answer
    \end{lstlisting}
    \end{tcolorbox}
    \caption{System prompt used for mathematical reasoning tasks.}
    \label{fig:prompt_math}
\end{figure*}

\begin{figure*}[t]
\centering
\begin{tcolorbox}[
    colback=white,
    colframe=black!70,
    title=\textbf{Unified Reflection Prompt Template},
    fonttitle=\bfseries\small,
    width=\textwidth,
    boxrule=0.5pt
]
\begin{lstlisting}[basicstyle=\ttfamily\scriptsize, breaklines=true]
You are a Reflector that analyzes trajectory logs based on user and environment feedback. 
Your goal is to identify what went wrong, trace root causes, and extract reusable principles 
for future improvement. Through Socratic-style iterative "why" questioning, trace issues back 
to their fundamental flawed assumptions or mental models.

Please output in the following JSON format:

{
"trajectory_summary": "Concise overview covering: (1) strategy employed, (2) final result, 
    (3) key observations about execution quality.",
"root_cause_analysis": "Deep causal analysis using iterative 'why' questioning to trace 
    from observable symptoms to the fundamental root cause. Chain reasoning explicitly.",
"trajectory_outcome": "One of: 'success', 'success_but_inefficient', 'failure'",
"improvement_suggestion": "Generalizable, context-complete principle for avoiding similar 
    issues. Must be self-contained and actionable.",
"retry_from_step": "Integer identifying the earliest step where the root cause first 
    manifested. Use 0 when root cause traces to initial strategy selection."
}
\end{lstlisting}
\end{tcolorbox}
\caption{The unified reflection prompt template used across all tasks.}
\label{fig:reflect_prompt_full}
\end{figure*}

\begin{figure*}[t]
\centering
\begin{tcolorbox}[
    colback=white,
    colframe=black!70,
    title=\textbf{Retry Guidance Template},
    fonttitle=\bfseries\small,
    width=\textwidth,
    boxrule=0.5pt
]
\begin{lstlisting}[basicstyle=\ttfamily\scriptsize, breaklines=true]
Your previous attempt encountered issues. Below is a reflection based on user and 
environment feedback:

{
    "trajectory_summary": "...",
    "root_cause_analysis": "...",
    "trajectory_outcome": "...",
    "improvement_suggestion": "...",
    "retry_from_step": ...
}

Apply the lessons learned from this reflection to avoid repeating the same mistakes.
Do not mention or reference this guidance in your response.
\end{lstlisting}
\end{tcolorbox}
\caption{The guidance prompt template used during the retry phase. The full JSON output from the reflection step is embedded into this template.}
\label{fig:guidance_prompt}
\end{figure*}

\end{document}